\documentclass[twoside,11pt]{article}
\usepackage{jair, theapa, rawfonts}

\jairheading{?}{2016}{??-??}{4/16}{?/??}
\ShortHeadings{Bayesian Network Structure Learning with Integer Programming}
{Cussens, J\"arvisalo, Korhonen, \& Bartlett}
\firstpageno{1}

\usepackage{amsthm, amsmath, amssymb}
\usepackage{tikz}
\usepackage{amsfonts}
\usepackage{bbm}
\usepackage{rotating}

\usepackage{url}

\newcommand{\vari}{\ensuremath{i}}
\newcommand{\varj}{\ensuremath{j}}
\newcommand{\vark}{\ensuremath{k}}
\newcommand{\varl}{\ensuremath{l}}
\newcommand{\digraph}{\ensuremath{(\vertices,\arcs)}}
\newcommand{\scorearcs}{\ensuremath{c(\arcss)}}
\newcommand{\digraphname}{\ensuremath{D}}
\newcommand{\vertices}{\ensuremath{V}}
\newcommand{\paset}[2]{\ensuremath{\mathrm{Pa}(#1,#2)}}

\newcommand{\localscorevec}{\ensuremath{c}}
\newcommand{\localscore}[2]{\ensuremath{c_{#1 \leftarrow #2}}}
\newcommand{\localscoredash}[2]{\ensuremath{c'_{#1 \leftarrow #2}}}
\newcommand{\varsubsetj}{\ensuremath{J}}

\newcommand{\pps}[1]{\ensuremath{{\cal P}(#1)}}
\newcommand{\ppsv}[1]{\ensuremath{{\cal P}_{\vertices}(#1)}}
\newcommand{\ppsdash}[1]{\ensuremath{{\cal P}'(#1)}}

\newcommand{\ppsalone}{\ensuremath{{\cal P}}}
\newcommand{\ppsvalone}{\ensuremath{{\cal P}_{\vertices}}}
\newcommand{\ppsdashalone}{\ensuremath{{\cal P}'}}
\newcommand{\ppsvdashalone}{\ensuremath{{\cal P}_{\vertices'}}}

\newcommand{\cluster}{\ensuremath{C}}

\newcommand{\arcs}{\ensuremath{A}}
\newcommand{\palim}{\ensuremath{\kappa}}
\newcommand{\arcss}{\ensuremath{B}}

\newcommand{\arcsets}{\ensuremath{\mathbf{F}}}

\newcommand{\arcsetsfull}{\ensuremath{{\cal F}}}

\newcommand{\digraphh}{\ensuremath{(\vertices,\arcss)}}

\newcommand{\arcweight}[2]{\ensuremath{c(#1 \leftarrow #2)}}

\newcommand{\reals}{\ensuremath{\mathbb{R}}}
\newcommand{\integers}{\ensuremath{\mathbb{Z}}}

\newcommand{\asppoly}[1]{\ensuremath{P_{\mathrm{AC}}(#1)}}
\newcommand{\fvpoly}[1]{\ensuremath{P_{\mathrm{F}}(#1)}}
\newcommand{\fvpolyy}[1]{\ensuremath{P_{\mathrm{F}}(#1)}}
\newcommand{\fvpolyyalone}{\ensuremath{P_{\mathrm{F}}}}

\newcommand{\fvpolyyo}[1]{\ensuremath{P_{\mathrm{F}}(#1,<)}}
\newcommand{\fvpolyys}[2]{\ensuremath{P_{\mathrm{F}}(#1,#2)}}
\newcommand{\assppolyalone}{\ensuremath{P_{\mathrm{G}}}}
\newcommand{\assppoly}[1]{\ensuremath{P_{\mathrm{G}}(#1)}}
\newcommand{\ascppoly}[1]{\ensuremath{P_{\mathrm{CLUSTER}}(#1)}}
\newcommand{\ascppolyalone}{\ensuremath{P_{\mathrm{CLUSTER}}}}

\newcommand{\lpsol}{\ensuremath{x^{*}}}

\newcommand{\coeffs}{\ensuremath{\pi}}
\newcommand{\rhs}{\ensuremath{\pi_{0}}}

\newcommand{\ndags}{\ensuremath{t}}

\newcommand{\dagindex}{\ensuremath{\iota}}

\newcommand{\card}[1]{\left\lvert {#1} \right\rvert}

\newcommand{\gobnilp}{{\sc gobnilp}}

\newcommand{\indic}[1]{\ensuremath{\mathbbm{1}(#1)}}

\def\conv{\mbox{\rm conv}\,}

\newtheorem{theorem}{Theorem}
\newtheorem{lemma}[theorem]{Lemma}

\newtheorem{proposition}[theorem]{Proposition}
\newtheorem{remark}[theorem]{Remark}

\theoremstyle{definition}
\newtheorem{definition}[theorem]{Definition}

\begin{document}

\title{Bayesian Network Structure Learning with
  Integer Programming:
Polytopes, Facets, and Complexity}

\author{\name James Cussens \email james.cussens@york.ac.uk\\
       \addr Department of Computer Science\\
       \& York Centre for Complex Systems Analysis\\
       University of York, United Kingdom
       \AND
       \name Matti J\"arvisalo \email matti.jarvisalo@helsinki.fi\\
       \addr Helsinki Institute for Information Technology HIIT\\
       Department of Computer Science\\
       University of Helsinki, Finland
       \AND
       \name Janne H. Korhonen \email  janne.h.korhonen@aalto.fi\\
       \addr  Department of Computer Science\\
       Aalto University, Finland
       \AND
       \name Mark Bartlett \email mark.bartlett@york.ac.uk \\
       \addr Department of Computer Science\\
       University of York, United Kingdom
}


\maketitle


\begin{abstract}
The challenging task of learning structures of probabilistic
graphical models is an important problem within modern AI research.
Recent years have witnessed several major algorithmic advances in 
structure learning for Bayesian networks---arguably
the most central class of graphical models---especially in what is
known as the
 score-based setting.
A successful generic
approach to optimal Bayesian network
structure learning (BNSL), based on  integer programming (IP),
is implemented in the \gobnilp{} system. 
Despite the recent algorithmic advances, current understanding of
foundational aspects underlying the IP based approach to BNSL is
still somewhat lacking.
Understanding fundamental aspects of \emph{cutting planes} and the
related \emph{separation problem}
is important not only from a purely theoretical perspective, but also since it holds out the promise of further improving the efficiency of state-of-the-art approaches to solving BNSL exactly. 
In this paper, we make several theoretical contributions towards these goals:
(i) we study the computational complexity of the separation problem, 
proving that the problem is NP-hard;
(ii) we formalise and analyse the relationship between three key polytopes underlying the IP-based approach to BNSL;
(iii) we study the facets of the three polytopes both from the theoretical and practical perspective, providing, 
via exhaustive computation, a complete enumeration of facets for low-dimensional \emph{family-variable} polytopes; and, furthermore,
(iv)~we establish a tight connection of the BNSL problem to the acyclic subgraph problem.
\end{abstract}

\section{Introduction}
\label{sec:intro}

The study of probabilistic graphical models is a central topic in 
modern artificial intelligence research.
Bayesian networks~\shortcite{koller09:_probab_graph_model} form a central
class of probabilistic graphical models that finds applications in
various domains~\shortcite{a/s16:_hugin,sheehan14:_improv_maxim_likel_recon_compl_multi_pedig}.  
A central problem related to Bayesian
networks (BNs) is that of learning them from data. An essential part of this
learning problem is to aim at learning the \emph{structure} of a Bayesian
network---represented as a directed acyclic graph---that accurately represents
the (hypothetical) joint probability distribution underlying the data.

There are two principal approaches to Bayesian network learning:
\emph{constraint-based} and \emph{score-based}. In the
constraint-based approach
\shortcite{spirtes93:_causat_predic_searc,colombo12:_learn} the goal
is to learn a network which is consistent with conditional
independence relations which have been inferred from the data.
The \emph{score-based} approach to Bayesian network structure
learning (BNSL) treats the BNSL problem as a combinatorial optimization problem of finding a BN structure
that optimises a score function for given data.  

Learning an optimal
BN structure is a computationally challenging
problem:
even the restriction 
of the BNSL problem where only 
\emph{BDe} scores~\shortcite{heckerman95:_learn_bayes} are allowed
is known to  be NP-hard~\cite{chickering96:_learn_bayes_networ_np_compl}.
Due to NP-hardness,
much work on BNSL has focused on developing approximate,
local search style algorithms~\shortcite{tsamardinos06:_bayes} that 
in general cannot  guarantee that optimal structures
 in terms of the objective function are found.
Recently,
despite its complexity, several advances in \emph{exact} approaches to
BNSL have
surfaced~\shortcite{koivisto04:_exact_bayes_struc_discov_bayes_networ,DBLP:conf/uai/SilanderM06,cussens11:_bayes_networ_learn_cuttin_planes,campos11:_effic_struc_learn_bayes_networ_const,yuan13:_learn_optim_bayes_networ,DBLP:conf/cp/BeekH15}, 
ranging from problem-specific
dynamic programming branch-and-bound algorithms to approaches based on
A$^*$-style state-space search, constraint programming, and
integer linear programming (IP), which can, with certain
restrictions, learn
provably-optimal BN structures with tens to hundreds of nodes.  

As
shown in a recent
study~\shortcite{malone14:_predic_hardn_learn_bayes_networ}, perhaps
the most successful exact approach to BNSL is provided by the
\gobnilp{}
system~\cite{cussens11:_bayes_networ_learn_cuttin_planes}. \gobnilp{}
implements a \emph{branch-and-cut} approach to BNSL, using state-of-the-art
IP solving techniques together with specialised BNSL cutting
planes. The focus of this work is on providing further understanding
of the IP approach to BNSL from the theoretical perspective.

Viewed as a constrained optimization problem, a central
source of intractability of BNSL is the \emph{acyclicity} constraint imposed on
BN structures. In the IP approach to BNSL---as implemented by \gobnilp---the acyclicity
constraint is handled in the branch-and-cut framework via 
deriving specialised cutting planes called \emph{cluster constraints}.
These cutting planes
 are found by solving a sequence of so-called \emph{sub-IPs} arising
from solutions to linear relaxations of the underlying IP formulation of BNSL without the
acyclicity constraint. Finding these cutting planes is an example of a
\emph{separation problem} for a linear relaxation solution, so called
since the cutting plane will separate that solution from the set of
feasible solutions to the original (unrelaxed) problem.
Understanding fundamental aspects of these cutting planes and the
sub-IPs used to find them is important not only from a purely theoretical perspective, 
but also since it holds out the promise
of further improving the efficiency of state-of-the-art approaches to solving BNSL exactly.
This is the focus of and underlying motivation for this article.

The main contributions of this article are the following.
\begin{itemize}
\item  We study the computational complexity of the separation problem
  solved via sub-IPs
 with connections to the general separation problem for integer programs.
As a main result, in Section~\ref{sec:complexity} we establish that the sub-IPs are themselves NP-hard to solve. From the practical
perspective, this both gives a theoretical justification for applying an exact IP solver to solve
the sub-IPs within \gobnilp, as well as motivates further work on improving the efficiency of the sub-IP solving
via either improved exact techniques and/or further approximate algorithms.
\item We formalise and analyse the relationship between three key polytopes underlying the
IP-based approach to BNSL in Section~\ref{sec:polytopes}.
Stated in generic abstract terms, starting from the \emph{digraph polytope}
defined by a linear relaxation of the  IP formulation without the acyclicity constraint,
the search progresses towards an optimal BN structure via refining the digraph polytope towards
the  \emph{family-variable polytope}, i.e., the convex hull of acyclic digraphs over the set of nodes in question.
The complete set of cluster constraints gives rise
to the \emph{cluster polytope} as an intermediate.
\item We study the \emph{facets} of the three  polytopes both from the theoretical and practical perspective (Section~\ref{sec:facets}). 
As a key theoretical result, we show that  cluster constraints are in
fact facet-defining inequalities of the family-variable polytope.
From the more practical perspective, achieved via exhaustive computation, we provide a complete enumeration of
facets for low-dimensional family-variable polytopes. Mapping to practice, explicit knowledge of 
such facets has the potential for providing further speed-ups in state-of-the-art BNSL solving by integrating
(some of) these facets explicitly into search.
\item In
Section~\ref{sec:faces} we derive facets of polytopes corresponding to
(i) BNs consistent with a given node ordering and (ii) BNs with
specified sink nodes. We then use the results on sink nodes to show
how a family-variable polytope for $p$ nodes 
can be constructed from a family-variable polytope for $p-1$ nodes
using the technique of \emph{lift-and-project}. 
\item Finally, in Section~\ref{sec:asp} we provide a tight connection of the BNSL problem to the  \emph{acyclic subgraph problem}, as well as discussing the connection of the polytope underlying this problem to the three central polytopes underlying BNSL.
\end{itemize}

Before detailing the main contributions, we recall the BNSL problem in Section~\ref{sec:bnsl}
and discuss the integer programming based approach to BNSL, central to this work, in Section~\ref{sec:ip}.

\section{Bayesian Network Structure Learning}
\label{sec:bnsl}

In this section, we recall the problem of learning optimal Bayesian network structures in the central score-based setting.

\subsection{Bayesian Networks}

A Bayesian network represents a joint probability distribution over a
set of random variables $Z = (Z_{\vari})_{\vari \in \vertices}$. A Bayesian network consists of a
\emph{structure} and \emph{parameters}:
\begin{itemize}
    \item The \emph{structure} is an acyclic digraph $\digraphh$ over
      the node set \vertices. For edge $\vari
\leftarrow \varj \in \arcss$ we say that
\vari{} is a \emph{child} of \varj{} and \varj{} is a \emph{parent} of
\vari, and for a variable $\vari \in \vertices$, we denote the set of parents of \vari{} by \paset{\vari}{\arcss}.
    \item The \emph{parameters} define a distribution for each of the random variables $Z_\vari$ for $\vari \in V$ conditional on the values of the parents, that is, the values \[\operatorname{Pr}\bigl( Z_{\vari} = z_{\vari} \mid Z_\varj = z_\varj \text{ for } \varj \in \paset{\vari}{\arcss}\bigr)\,.\]
\end{itemize}
The joint probability distribution of the Bayesian network is defined in terms of the structure and the parameters as
\[\operatorname{Pr}(Z_\vari = z_\vari \text{ for } \vari \in \vertices) = \prod_{\vari \in \vertices} \operatorname{Pr}\bigl( Z_{\vari} = z_{\vari} \mid Z_\varj = z_\varj \text{ for } \varj \in \paset{\vari}{\arcss}\bigr)\,.\]

As mentioned before, our focus is on learning Bayesian networks from
data. Specifically, we focus on the Bayesian network structure
learning (BNSL) problem. Once a BN structure has been decided, its
parameters can be learned from the data. See, for example,
\citeA{koller09:_probab_graph_model} on techniques for parameter
estimation for a given BN structure.

\subsection{Score-based BNSL}

In the integer programming based approach to BNSL which is the focus of this
work, the learning problem is cast as a constrained optimisation
problem: each candidate BN structure has a score measuring how well it
`explains' the given data and the task is to find a BN structure which maximises
that score. This score function is defined in terms of the data, but for
our purposes, it is sufficient to abstract away the details, see e.g.\ \citeA{koller09:_probab_graph_model}.

Specifically, in this paper we restrict attention to \emph{decomposable} score
functions, where the score is defined locally by the parent set choices for
each $\vari \in \vertices$. Specifically, for $\vari \in V$ and
$\varsubsetj \subseteq \vertices \setminus \{\vari\}$, let $\vari \leftarrow \varsubsetj$
denote the the pair $(\vari, \varsubsetj)$, called a
\emph{family}.
In our framework, we assume that the score function gives a \emph{local score} $\localscore{\vari}{\varsubsetj}$
for each family $\vari \leftarrow \varsubsetj$. A global score $\scorearcs$ for each candidate
structure \digraphh{} is then defined as
\begin{equation}
  \label{eq:dagscore}
\scorearcs = \sum_{\vari \in \vertices} \localscore{\vari}{\paset{\vari}{\arcss}},
\end{equation}
and the task to find an acyclic digraph \digraphh{} maximising $\scorearcs$
over all acyclic digraphs over $V$.

In practice, one may  want to restrict the set of parent sets in some way,
given the large number of possible parents sets and the NP-hardness of BNSL.
Typically this is done by limiting the cardinality of each candidate parent set, although other
restrictions, perhaps reflecting prior knowledge, can also be used. To facilitate this,
we assume that a BNSL instance also defines a set of permissible parent sets
$\pps{\vari} \subseteq 2^{\vertices \setminus \{\vari\}}$ for each node
$\vari$. For simplicity we shall only consider BNSL problems where
$\emptyset \in \pps{\vari}$ for all nodes. This also ensures that the empty
graph, at least, is a permitted BN structure.
Thus, the full formulation of the BNSL problem is as follows.

\begin{definition}[BNSL] A \emph{BNSL instance} is a tuple $(V,\ppsalone, c)$, where
\begin{enumerate}
\item $\vertices$ is a set of nodes;
\item $\ppsalone \colon \vertices \rightarrow 2^{2^{\vertices}}$ is a function
  where, for each vertex $\vari \in \vertices$, $\pps{\vari} \subseteq
  2^{\vertices \setminus \{\vari\}}$ is the set of permissible
  parent sets for that vertex, and $\emptyset \in
  \pps{\vari}$; and 
\item $c$ is a function giving the local score
  $\localscore{\vari}{\varsubsetj}$ for each $\vari \in \vertices$ and
  $\varsubsetj \in \pps{\vari}$.\end{enumerate}
Given a BNSL instance $(V,\ppsalone, c)$, the \emph{BNSL problem} is to find an edge set $\arcss \subseteq \vertices \times \vertices$ which
maximises (\ref{eq:dagscore})
subject to the following two conditions.
\begin{enumerate}
\item $\paset{\vari}{\arcss} \in \pps{\vari}$ for all $\vari \in
  \vertices$.
\item $\digraphh$ is acyclic.
\end{enumerate}
\end{definition}

\subsection{BNSL with Small Parent Sets}
\label{sec:palim}

As mentioned, it is common to put an upper bound on the
cardinality of permitted parent sets. More precisely, a common setting is that
 we have a constant $\palim$
and the BNSL instances we consider are restricted so that all $\varsubsetj \in \pps{\vari}$
satisfy $\card{\varsubsetj} \le \palim$. For the rest of the paper we use the convention 
that $\palim$ denotes this upper bound on parent set size.

In practice, BNSL instances
with large node set size can often be solved to optimality fairly quickly when $\palim$ is small.
For example, with $\palim=2$, 
\shortciteA{sheehan14:_improv_maxim_likel_recon_compl_multi_pedig} were able
to solve BNSL instances with $\card{\vertices} = 1614$ in between 3 and 42 minutes.
Even though BNSL remains NP-hard unless $\palim = 1$ \shortcite{chickering96:_learn_bayes_networ_np_compl}, such results suggest that \emph{in practice} the value of $\palim$ is an important
determining factor of the hardness of a BNSL instance.

However, we will show in the following that the situation is somewhat more subtle:
we show that any BNSL instance can be converted to a BNSL instance with
$\palim=2$ and the same set of optimal solutions without significantly
increasing the total size $\card{\vertices} + \sum_{\vari \in \vertices} \card{\pps{\vari}}$ of the instance.
This suggests, to a degree, that this total instance size is an important control parameter for 
the hardness of BNSL instances; naturally, with
larger $\palim$, a smaller number of nodes is required for a large total size.\footnote{The conversion to
a BNSL instance with $\palim=2$ presented here may influence
the  runtime performance of BNSL solvers in practice. For example,
we have observed through experimentation that the 
runtime performance of the \textsc{gobnilp} system
often degrades if the conversion is applied before
search.} 

We first introduce some useful notation identifying the set of
families in a BNSL instance.
For a given set $\vertices$ of
nodes and permitted parent sets $\pps{\vari}$, let
\[
  \arcsetsfull(\vertices,\ppsalone)  :=  \{\vari 
\leftarrow \varsubsetj \mid \vari \in \vertices, \varsubsetj \in \pps{\vari} \},\]
so that $\sum_{\vari \in \vertices} \card{\pps{\vari}} =
\card{\arcsetsfull(\vertices,\ppsalone)}$ and total instance size is
$\card{\vertices} + \card{\arcsetsfull(\vertices,\ppsalone)}$.

\begin{theorem}\label{thm:bnslp-bounded-parent-sets}
Given a BNSL instance $(\vertices, \ppsalone, \localscorevec)$ with
the property that for each $\vari \in \vertices$, $\pps{\vari}$ is downwards-closed, that is, $I \subseteq \varsubsetj \in \pps{\vari}$ implies $I \in \pps{\vari}$, we can construct another BNSL instance $(\vertices', \ppsalone', \localscorevec')$ in time $\operatorname{poly}\bigl(\card{\vertices} + \card{\arcsetsfull(\vertices,\ppsalone)}\bigr)$ such that 
\begin{enumerate}
    \item $\card{\vertices'} = O\bigl(\card{\vertices} + \card{\arcsetsfull(\vertices,\ppsalone)}\bigr)$ and $\card{\arcsetsfull(\vertices',\ppsalone')} = O\bigl(\card{\arcsetsfull(\vertices,\ppsalone)}\bigr)$,
    \item $\card{J} \le 2$ for all $\varsubsetj \in \ppsdash{i}$ and $\vari \in \vertices'$, and
    \item there is one-to-one correspondence between the optimal solutions of $(\vertices, \ppsalone, \localscorevec)$ and $(\vertices', \ppsalone', \localscorevec')$. 
\end{enumerate}
Moreover, the claim holds even when $(\vertices, \ppsalone,
\localscorevec)$ does not satisfy the downwards-closed property, with
bounds $\card{\vertices'} = O\bigl(\card{\vertices} +
\kappa\card{\arcsetsfull(\vertices,\ppsalone)}\bigr)$ and
$\card{\arcsetsfull(\vertices',\ppsalone')} =
O\bigl(\kappa\card{\arcsetsfull(\vertices,\ppsalone)}\bigr)$, where
$\kappa$ is the size of the largest parent set permitted by 
$\ppsalone$.
\end{theorem}

\begin{figure}
  \includegraphics[width=\linewidth]{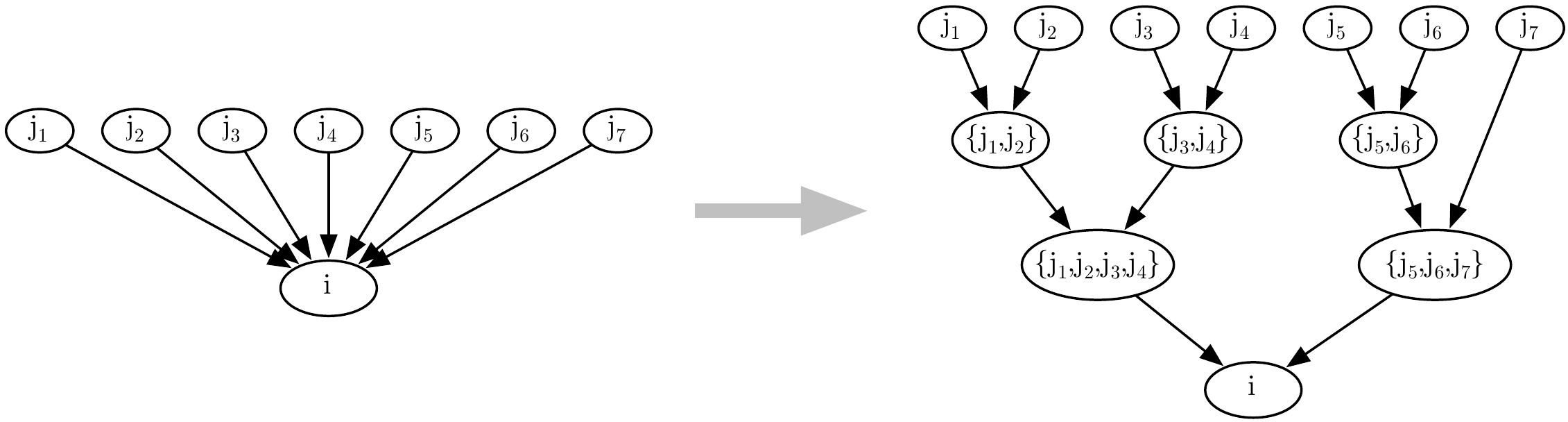}
  \caption{The basic idea of the reduction in Theorem~\ref{thm:bnslp-bounded-parent-sets}. Selecting the parent set $\{ j_1, j_2, j_3, j_4, j_5, j_6, j_7 \}$ for node $i$ in the original instance corresponds to selecting the parent set $\bigl\{ \{ j_1, j_2, j_3, j_4 \}, \{ j_5, j_6, j_7 \}\bigr\}$ in the transformed instance. Note that the parent sets for the nodes labelled with sets are fixed.}
  \label{fig:thm2}
\end{figure}

\begin{proof}
Given $(V,\ppsalone,c)$, we construct a new instance $(V',\ppsalone',c')$ as follows.
As a first step, we iteratively go through the permissible parent sets $\varsubsetj \in \pps{\vari}$ for each $\vari \in \vertices$ and add the corresponding new parent set to $\ppsdash{\vari}$ using the following rules; Figure~\ref{fig:thm2} illustrates the basic idea.
\begin{itemize}
    \item If $\card{\varsubsetj} \le 2$, we add $\varsubsetj$ to $\ppsdash{\vari}$ with score $\localscoredash{\vari}{\varsubsetj} = \localscore{\vari}{\varsubsetj}$.
    \item If $\varsubsetj = \{ \varj, \vark, \varl \}$, then we create a new node $I \in \vertices'$ corresponding to the subset $I = \{ \vark, \varl \}$, and add the set $\varsubsetj' = \{ \varj, I \}$ to $\ppsdash{\vari}$ with score $\localscoredash{\vari}{\varsubsetj'} = \localscore{\vari}{\varsubsetj}$.
    \item If $\card{\varsubsetj} \ge 4$, we partition $\varsubsetj$ into two sets $\varsubsetj_1$ and
  $\varsubsetj_2$ with $\card{|\varsubsetj_{1}| - |\varsubsetj_{2}|} \leq 1$ and create new corresponding nodes $\varsubsetj_1, \varsubsetj_2 \in \vertices'$. We then add $\varsubsetj' = \{ \varsubsetj_1, \varsubsetj_2 \}$ to $\ppsdash{\vari}$ with score $\localscoredash{\vari}{\varsubsetj'} = \localscore{\vari}{\varsubsetj}$.
\end{itemize}
In the above steps, new nodes corresponding to subsets of $\vertices$
will be created only once, re-using the same node if it is required multiple times.

Unless all original parent sets have size at most two, this process will create new nodes $\varsubsetj \in \vertices'$ corresponding to subsets $\varsubsetj \subseteq \vertices$ with $\card{\varsubsetj} \geq 2$. For each 
such new node $\varsubsetj$, we allow exactly one permissible parent set (of size $2$) besides the empty set, as follows.
\begin{itemize}
  \item If $\varsubsetj = \{\varj, \vark\}$, then set $\ppsdash{\varsubsetj}
    = \{ \emptyset, \{\varj, \vark\} \}$.
  \item If $\varsubsetj = \{\varj, \vark, \varl\}$, then set $\ppsdash{\varsubsetj}
    = \bigl\{ \emptyset, \{\varj, \{\vark,\varl\} \} \bigr\}$, choosing $\varj$ arbitrarily and creating a new node $\{\vark,\varl\}$ if necessary.
  \item If $\card{\varsubsetj} \ge 4$, then we partition $\varsubsetj$ into some $\varsubsetj_1$ and
    $\varsubsetj_2$ where $\card{\card{\varsubsetj_{1}} - \card{\varsubsetj_{2}}} \leq 1$
    and set $\ppsdash{\varsubsetj} = \{ \emptyset, \{ \varsubsetj_{1}, \varsubsetj_{2} \} \}$, again creating new nodes $\varsubsetj_1$ and $\varsubsetj_2$ if necessary.
  \end{itemize}
However, we want to disallow the choice of $\emptyset$ for all new
nodes in all optimal solutions, so we will set $\localscoredash{\varsubsetj}{\emptyset} =  \min\bigl({-\card{\vertices}}, M\card{\vertices}\bigr)$, where $M$ is the minimum score given to any family by $c$, and set the local score for the other parent set choices to $0$.
  
The creation of these parent sets may require the creation of yet further new nodes. If so, we create the permissible parent sets for each of them in the same way, iterating the process as long as necessary. This will clearly terminate, and if $(\vertices, \ppsalone, \localscorevec)$ satisfies the downwards-closed property, this will create exactly one new node in $\vertices'$ for each original permissible parent set, implying the bounds for $\card{\vertices'}$ and $\card{\arcsetsfull(\vertices',\ppsalone')}$. If the original instance does not have the downwards-closed property, the process may create up to $\card{\varsubsetj}$ new nodes for each original $\varsubsetj \in \pps{\vari}$, which in turn implies the weaker bound.

Finally, note that any optimal solution to $(\vertices', \ppsalone',
\localscorevec')$ cannot pick the empty set as a parent set for a node corresponding to a subset of $\vertices$. It is now not difficult to see that, from any optimal solution to our newly created BNSL instance, we can `read off' an optimal solution to the original  instance.
\end{proof}

\section{An Integer Programming Approach to Bayesian Network Structure Learning}
\label{sec:ip}

In this section we discuss integer programming based approaches to BNSL, focusing on the
branch-and-cut approach implemented by the 
\textsc{gobnilp} system for BNSL which motivates the theoretical results presented in this article.

\subsection{An Integer Programming Formulation of BNSL}

Recall, from Section~\ref{sec:bnsl}, that we refer to a node $\vari$
together with its parent set $\varsubsetj$ as a \emph{family}. In the
IP formulation of BNSL we create a \emph{family variable} $x_{\vari
  \leftarrow \varsubsetj}$ for each potential family. A family
variable is a binary indicator variable: $x_{\vari \leftarrow  \varsubsetj} = 1$ if $\varsubsetj$ is 
 the parent set for
$\vari$ and $x_{\vari \leftarrow  \varsubsetj} = 0$ otherwise. It is not difficult to see that any digraph
(acyclic or otherwise) with $|\vertices|$ nodes can be encoded by a
zero-one vector whose components are family variables and where
exactly $|\vertices|$ family variables are set to
1. Figure~\ref{fig:bnex} and Table~\ref{tab:bnex} show an example graph
and its family variable encoding, respectively.

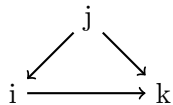
\begin{figure}
  \centering
   \begin{tikzpicture}
      \node (a) at (0,2) {i};
      \node (b) at (1,3) {j};
      \node (c) at (2,2) {k};
      \draw[->,thick] (a) -- (c);
      \draw[->,thick] (b) -- (c);
      \draw[->,thick] (b) -- (a);
    \end{tikzpicture}
  \caption{A digraph with 3 nodes.}
  \label{fig:bnex}
\end{figure}

\begin{table}
  \centering
\begin{tabular}{|l|l|l|l|}
    \hline
    $i\leftarrow\{\}$ & $i\leftarrow\{j\}$ & $i\leftarrow\{k\}$ &
    $i\leftarrow\{j,k\}$\\ \hline
    0 & 1 & 0 & 0  \\ \hline\hline
    $j\leftarrow\{\}$ & $j\leftarrow\{i\}$ & $j\leftarrow\{k\}$ &
    $j\leftarrow\{i,k\}$\\ \hline
    1 & 0 & 0 & 0  \\ \hline \hline
    $k\leftarrow\{\}$ & $k\leftarrow\{i\}$ & $k\leftarrow\{j\}$ &
    $k\leftarrow\{i,j\}$\\ \hline\hline
    0 & 0 & 0 & 1  \\
    \hline 
  \end{tabular}
  \caption{A vector in $\reals^{12}$ which is the family variable
    encoding of the digraph in Figure~\protect\ref{fig:bnex} where all
    possible parent sets are permitted. Here each of the
  12 components is labelled with the appropriate family and the
  vector is displayed in three rows.}
  \label{tab:bnex}
\end{table}

Although every digraph can thus be encoded as a zero-one vector, it is
clearly not the case that each zero-one vector encodes a digraph. The
key to the IP approach to BNSL is to add appropriate linear
constraints so that all and only zero-one vectors representing acyclic digraphs
satisfy all the constraints.

The most basic constraints are illustrated by the arrangement of the example vector in
Table~\ref{tab:bnex} into three rows, one for each
node. It is clear that exactly one family variable for each child
node must equal one.  So we have $|\vertices|$ \emph{convexity constraints}
\begin{equation}
  \label{eq:convexitya}
  \sum_{\varsubsetj \in \pps{\vari}} x_{\vari \leftarrow
  \varsubsetj} = 1 \quad \forall \vari \in \vertices,
\end{equation}
each of which may have an exponential number of terms.
It is not difficult to see that any vector $x$ that satisfies all convexity
constraints encodes a digraph. However, without further constraints,
the digraph need not be acyclic. There are a number of ways of ruling
out cycles
\shortcite{cussens10:_maxim,peharz12:_exact_maxim_margin_struc_learn_bayes_networ,cussens13:_maxim_likel_pedig_recon_integ_linear_progr}.
In this paper we focus on \emph{cluster constraints} first
introduced by 
\shortciteA{jaakkola10:_learn_bayes_networ_struc_lp_relax}. A \emph{cluster}
is simply a subset of nodes with at least 2 elements. For each cluster $\cluster \subseteq
\vertices$ ($|\cluster|>1$) the associated cluster inequality is
\begin{equation}
  \label{eq:clusterineq}
  \sum_{\vari\in\cluster} \ \sum_{\varsubsetj \in \pps{\vari}:\varsubsetj
    \cap \cluster = \emptyset}
  x_{\vari \leftarrow \varsubsetj} \geq 1.
\end{equation}
An alternative formulation, which exploits the convexity constraints,
is
\begin{equation}
  \label{eq:clusterineqalt}
  \sum_{\vari\in\cluster}  \ \sum_{\varsubsetj \in \pps{\vari}:\varsubsetj
    \cap \cluster \neq \emptyset}
  x_{\vari \leftarrow \varsubsetj} \leq |\cluster| - 1.
\end{equation}

To see that cluster inequalities suffice to rule out cycles, note
that, for any cluster $\cluster$ and digraph $x$, the left-hand side (LHS) of
\eqref{eq:clusterineq} is a count
of the number of vertices in $\cluster$ that in $x$ have no parents in
$\cluster$. Now suppose that the nodes in some cluster $\cluster$
formed a cycle; it is clear that in that case the LHS
of \eqref{eq:clusterineq} would be 0, violating the cluster
constraint. On the other hand, suppose that $x$ encodes an acyclic
digraph. Since the digraph is acyclic, there is an associated total
ordering in which parents precede their children.  Let
$\cluster \subseteq \vertices$ be an arbitrary cluster. Then the
earliest element of $\cluster$ in this ordering will have no parents
in $\cluster$ and so the LHS of \eqref{eq:clusterineq} is at least 1
and the cluster constraint is satisfied. An illustration of how acyclic
graphs satisfy all cluster constraints and cyclic graphs do not is
given in Figure~\ref{fig:clusterex}.

\begin{figure}
  \centering
     \begin{tikzpicture}
      \node (c) at (0,0) {c};
      \node (d) at (1,0) {d};
      \node (a) at (0,1) {a};
      \node (b) at (1,1) {b};
      \draw[->,thick] (a) -- (c);
      \draw[->,thick] (b) -- (a);
      \draw[->,thick] (b) -- (c);
      \draw[->,thick] (c) -- (d);
    \end{tikzpicture}
    \hspace{2cm}
    \begin{tikzpicture}
      \node (c) at (0,0) {c};
      \node (d) at (1,0) {d};
      \node (a) at (0,1) {a};
      \node (b) at (1,1) {b};
      \draw[->,thick] (a) -- (c);
      \draw[->,thick] (b) -- (a);
      \draw[->,thick] (c) -- (d);
      \draw[->,thick] (d) -- (b);
    \end{tikzpicture}
    \caption{An acyclic and a cyclic graph for vertex set
      $\{a,b,c,d\}$. For each cluster of vertices $C$ where $|C|>1$
      let $f(C)$ be the number of vertices in $C$ who have no parents
      in $C$ (i.e.\ the LHS of
      \protect{\eqref{eq:clusterineq}}). Abbreviating e.g.\ $\{a,b\}$
      to $ab$, for the left-hand graph we have: $f(ab) = 1$,
      $f(ac) = 1$, $f(ad) = 2$, $f(bc) = 1$, $f(bd) = 2$,
      $f(cd) = 1$, $f(abc) = 1$, $f(abd) = 2$, $f(acd) =1$, $f(bcd) = 1$ and
      $f(abcd) = 1$. For the right-hand graph we have: $f(ab) = 1$,
      $f(ac) = 1$, $f(ad) = 2$, $f(bc) = 2$, $f(bd) = 1$, $f(cd) = 1$,
      $f(abc) = 1$, $f(abd) = 1$, $f(acd) =1$, $f(bcd) = 1$ and
      $f(abcd) = 0$.  The cluster constraint for cluster $\{a,b,c,d\}$
      is violated by the right-hand graph since these vertices form a
      cycle. }
  \label{fig:clusterex}
\end{figure}
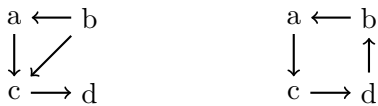

It follows that any zero-one vector $x$ that satisfies the convexity
constraints (\ref{eq:convexitya}) and cluster constraints
(\ref{eq:clusterineq}) encodes an acyclic digraph. The final ingredient
in the IP approach to BNSL is to specify objective coefficients for
each family variable. These are simply the local scores
$\localscore{\vari}{\varsubsetj}$ introduced in Section~\ref{sec:bnsl}.
Collecting these elements together, we can define the IP formulation of
the BNSL as follows.

\begin{align}
\mbox{\textsc{Maximise} }\quad  
& \sum_{\vari \in \vertices, \varsubsetj \in \pps{\vari}}
\localscore{\vari}{\varsubsetj} x_{\vari \leftarrow \varsubsetj}
& \label{eq:obj}\\
\mbox{\textsc{subject to}} \quad  
& \ \ \ \ \ \ \ \ \ \ \ \ \ \sum_{\varsubsetj \in \pps{\vari}} x_{\vari \leftarrow  \varsubsetj} = 1 
& \forall \vari \in \vertices \label{eq:convex} \\
& \sum_{\vari\in\cluster} \sum_{\varsubsetj \in \pps{\vari}:\varsubsetj \cap \cluster = \emptyset}
   x_{\vari \leftarrow \varsubsetj} \geq 1 
&\forall \cluster \subseteq \vertices,\ |\cluster|>1 \label{eq:cluster} 
 \\
& \ \ \  \ \ \ \ \ \quad \quad  \quad  \quad  \   x_{\vari \leftarrow \varsubsetj} \in \{0,1\}
& \forall \vari \in \vertices,\ \varsubsetj \in \pps{\vari} \label{eq:int}
\end{align}

\subsection{The \gobnilp{} System}
\label{sec:gobnilp}

The \gobnilp{} system (\url{https://www.cs.york.ac.uk/aig/sw/gobnilp/})
solves the IP problem defined by (\ref{eq:obj}--\ref{eq:int}) for a given set of objective
coefficients $\localscore{\vari}{\varsubsetj}$. These coefficients are
either given as input to \gobnilp{} or computed by \gobnilp{} from a
discrete dataset with no missing values. The \gobnilp{} approach to solving
this IP is fully detailed  by \shortciteA{Bartlett2015}; here we overview
the essential ideas.

Since there are only $|\vertices|$ convexity constraints
\eqref{eq:convex}, these are added as initial constraints to the
IP.
Initially, no cluster constraints \eqref{eq:cluster} are in the IP, so
we have a relaxed version of the original problem. Moreover, in its
initial phase \gobnilp{} relaxes the integrality condition
\eqref{eq:int} on the family variables into
$x_{\vari \leftarrow \varsubsetj} \in [0,1]$
$\forall \vari \in \vertices,\ \varsubsetj \in \pps{\vari}$, so that
only linear relaxations of IPs are solved. So \gobnilp{} starts with a
`doubly' relaxed problem: the constraints ruling out cycles are missing
and the integrality condition is also dropped.

A linear relaxation of an IP is a linear program (LP). \gobnilp{} uses
an external LP solver such as SoPlex or CPLEX to solve linear
relaxations.  The solution (call it $\lpsol$) to the initial LP will
be a digraph where a highest scoring parent set for each node is
chosen, a digraph which will almost certainly contain cycles. Note
that this initial solution happens to be integral, even though it is
the solution to an LP not an IP. \gobnilp{} then attempts to find
clusters $\cluster$ such that the associated cluster constraint is
violated by $\lpsol$. Since the cluster constraints are added in this
way they are called \emph{cutting planes}: each one \emph{cuts} off an
(infeasible) solution $\lpsol$ and since they are linear each one
defines a (high-dimensional) \emph{plane}. These cluster constraints
are added to the LP, producing a new LP which is then solved,
generating a new solution $\lpsol$. This process is illustrated in
Figure~\ref{fig:cuts}. However, since the cutting planes found by
\gobnilp{} are rather hard to visualise, we use a (non-BNSL) IP
problem with only two variables to illustrate the basic ideas behind
the cutting plane approach. Note that the relaxation in
Figure~\ref{fig:cuts} contains (infeasible) integer solutions. Many of
the relaxations solved by GOBNILP (notably the initial one) also allow
infeasible integer solutions---which correspond to cyclic digraphs.

\begin{figure}
  \centering

\begin{tikzpicture}


  \draw  (1,0) -- (2,2) -- (3,2) -- (3,0) -- (1,0);

  \draw [->] (-0.5,0) -- (-0.5,0.5) node[above] {$y$};
  \draw [->] (0,-0.5) -- (0.5,-0.5) node[right] {$x$};

  \fill[green,opacity=0.5] (0,0) -- (2.5,3) -- (4.8,0);



    \draw [color=blue] (0,2.6) -- (5,2.1);
    \draw [thick,->] (4.3,2.4) -- (3.8,2.9);

  \foreach \x in {0,1,2,3,4,5} 
  \foreach \y in {0,1,2,3} 
  {\filldraw [black] (\x,\y) circle (0.02cm);}

    \filldraw (1,0) circle (0.05cm);
  \filldraw (2,0) circle (0.05cm);
  \filldraw (3,0) circle (0.05cm);
  \filldraw (2,1) circle (0.05cm);
  \filldraw (3,1) circle (0.05cm);
  \filldraw[color=red] (2,2) circle (0.05cm);
  \filldraw  (3,2) circle (0.05cm);

  {\filldraw [color=orange] (2.5,3) circle (0.04cm);}


  \filldraw [color=yellow](2,2.4) circle (0.05cm);
  
\end{tikzpicture}\hspace{2cm}
\begin{tikzpicture}


      \draw  (1,0) -- (2,2) -- (3,2) -- (3,0) -- (1,0);

  \draw [->] (-0.5,0) -- (-0.5,0.5) node[above] {$y$};
  \draw [->] (0,-0.5) -- (0.5,-0.5) node[right] {$x$};

      \fill[green,opacity=0.5] (0,0) -- (2.5,3) -- (4.8,0);

    \draw [color=blue] (0,2) -- (5,2);
    \draw [thick,->] (4.3,2.4) -- (3.8,2.9);
    



  \foreach \x in {0,1,2,3,4,5} 
  \foreach \y in {0,1,2,3} 
  {\filldraw [black] (\x,\y) circle (0.02cm);}

  \filldraw (1,0) circle (0.05cm);
  \filldraw (2,0) circle (0.05cm);
  \filldraw (3,0) circle (0.05cm);
  \filldraw (2,1) circle (0.05cm);
  \filldraw (3,1) circle (0.05cm);
  \filldraw[color=red] (2,2) circle (0.05cm);
  \filldraw  (3,2) circle (0.05cm);

  {\filldraw [color=orange] (2.5,3) circle (0.04cm);}

  \filldraw [color=yellow](1.6666,2) circle (0.05cm);
  

\end{tikzpicture}
\caption{Illustration of the cutting plane technique for a problem
  with 2 integer-valued variables $x$ and $y$. \textbf{In both
    figures}: The 7 large dots indicate the 7 feasible solutions
  $(x=1,y=0)$, $(x=2,y=0)$, $(x=3,y=0)$, $(x=2,y=1)$, $(x=3,y=1)$,
  $(x=2,y=2)$ and $(x=3,y=2)$. The red dot indicates the optimal
  solution $(x=2,y=2)$. The objective function is $-x+y$ and is
  indicated by the arrow. The boundary of the convex hull of feasible
  solutions is shown. A relaxation of the problem is indicated by the
  green region with the orange dot indicating the optimal solution to
  the relaxed problem.  \textbf{In the left-hand figure}: The blue
  line represents a cutting plane---a linear inequality---which
  separates the solution to the relaxed problem from the convex hull
  of feasible solutions. The yellow dot indicates the optimal solution
  to the relaxed problem once this cut is added. \textbf{In the
    right-hand figure}: Similar to the left-hand figure except that a
  better cut has been added. The right-hand yellow dot has a lower
  objective value than that on the left and thus provides a better
  upper bound.}
  \label{fig:cuts}
\end{figure}
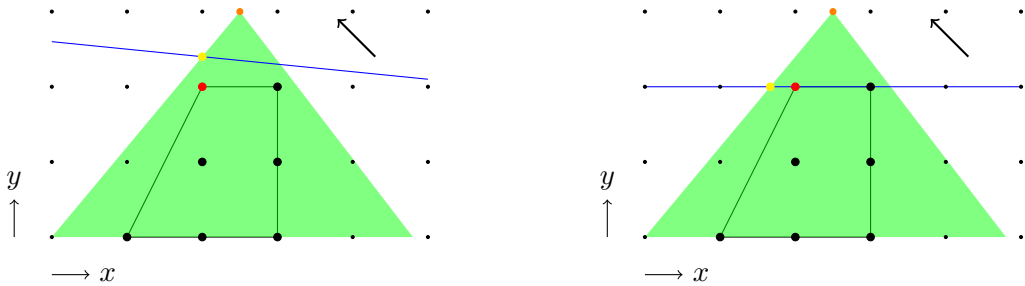

The process of LP solving and adding cluster constraint cutting planes
is continued until either (i) an LP solution is produced which
corresponds to an acyclic digraph, or (ii) this is not the case, but no
further cluster constraint cutting planes can be found.  In the first
(rare) case, the BNSL instance has been solved. The objective value of
each $\lpsol$ that is produced is an upper bound on the objective
value of an optimal digraph (since it is an exact solution to a
relaxed version of the original BNSL instance), so if $\lpsol$
corresponds to an acyclic digraph it must be optimal.

The second (typical) case can occur since even if we were to
add all (exponentially-many) cluster constraints to the LP there is no
guarantee that the solution to that LP would be integral. (This
hypothetical LP including all cluster constraints defines what we call the \emph{cluster
  polytope} which will be discussed in
Section~\ref{sec:clusterpoly}.) However, since we only add those cluster
constraints which are \emph{cutting planes} (i.e., which cut off the
solution $\lpsol$ to some linear relaxation) in practice only a small
fraction of cluster constraints are actually added.\footnote{We have
  yet to explore the interesting question of how large this fraction might be.}

Once no further cluster constraint cutting planes can be found
\gobnilp{} stops ignoring the integrality constraint \eqref{eq:int} on
family variables and exploits it to make progress. If no cluster
constraint cutting planes can be found, and the problem has not been
solved, then $\lpsol$, the solution to the current linear relaxation,
must be fractional, i.e., there must be at least one family variable
$x_{\vari \leftarrow \varsubsetj}$ such that
$0 < \lpsol_{\vari \leftarrow \varsubsetj} < 1$. One option is then to
\emph{branch} on such a variable to create two sub-problems: one where
$x_{\vari \leftarrow \varsubsetj}$ is fixed to 0 and one where it is
fixed to 1. Note that $\lpsol$ is infeasible in both sub-problems but
there is an optimal solution in at least one of the
sub-problems. \gobnilp{} also has the option of branching on sums of
mutually exclusive family variables. For example, given 
nodes $i$, $j$, and $k$, \gobnilp{} has the option of branching on
$x_{i\leftarrow\{j\}} + x_{i\leftarrow\{j,k\}} + x_{j\leftarrow\{i\}}
+ x_{j\leftarrow\{i,k\}}$, a quantity which is either 0 or 1 in an
acyclic digraph. \gobnilp{} then recursively applies the cutting plane
approach to both sub-problems. \gobnilp{} is thus a
\emph{branch-and-cut} approach to IP solving.

These are the essentials of the \gobnilp{} system, although the current
implementation has many other aspects.
In particular, under default parameter values, \gobnilp{} switches to
branching on a fractional variable if 
the search for cluster constraint
cutting planes is taking too long.
\gobnilp{} is implemented with the
help of the SCIP system (\url{http://scip.zib.de})
\cite{achterberg07:_const_integ_progr} and it uses SCIP to
generate many other cutting planes in addition to cluster
constraints. \gobnilp{} also adds in other initial inequalities in
addition to the convexity constraints. For example, if we had three
nodes $i$, $j$, and $k$, the inequality $x_{i\leftarrow\{j,k\}} +
x_{j\leftarrow\{i,k\}} + x_{k\leftarrow\{i,j\}} \leq 1$ would be
added. All these extra constraints are redundant in the sense that
they do not alter the set of optimal solutions to the IP
(\ref{eq:obj}--\ref{eq:int}). They do, however, have a great effect in the time
taken to identify a provably optimal solution.

\subsection{BNSL Cutting Planes via Sub-IPs}

The \emph{separation problem} for an IP is the problem of finding a
cutting plane which is violated by the current linear relaxation of
the IP, or to show that none exists. In this paper we focus on the
special case of finding a \emph{cluster constraint} cutting plane for
an LP solution $\lpsol$, or showing none exists. We call this the
\emph{weak separation problem}. We call it the `weak'
separation problem since cluster constraints are not the only possible
cutting planes.

In \gobnilp, this problem is solved via a sub-IP, as described earlier e.g.~by 
\citeA{Bartlett2015}.
 Given an LP solution $\lpsol$ to separate, the
variables of the sub-IP include binary variables $y_{\vari
  \leftarrow \varsubsetj}$ for each family such that $\lpsol_{\vari
  \leftarrow \varsubsetj} > 0$. In addition, binary variables
$y_{\vari}$ for each $\vari \in \vertices$ are created. The constraints of the sub-IP are such that $y_{\vari}=1$ 
indicates that $\vari$ is a member of some cluster whose associated
cluster constraint is a cutting plane for $\lpsol$. $y_{\vari
  \leftarrow \varsubsetj}=1$ indicates that the family variable $x_{\vari
  \leftarrow \varsubsetj}$ appears in the cluster constraint. The
sub-IP is given by

\begin{align}
\mbox{\textsc{Maximise}} \quad
& \sum_{\vari, \varsubsetj \ : \ \lpsol_{\vari\leftarrow \varsubsetj} > 0}
 \lpsol_{\vari \leftarrow \varsubsetj} \cdot y_{\vari \leftarrow \varsubsetj} -\sum_{\vari \in \vertices} y_{\vari} 
&\label{eq:subobj} \\
\mbox{\textsc{subject to}}  \quad
&  \quad \quad \quad \quad \quad \quad \quad y_{\vari \leftarrow \varsubsetj} \Rightarrow y_{\vari}
&  \forall y_{\vari \leftarrow \varsubsetj} \label{eq:subip1} \\
&  \quad \quad \quad \quad \quad \quad \quad y_{\vari \leftarrow \varsubsetj} \Rightarrow \bigvee_{\varj \in \varsubsetj} y_{\varj} 
& \forall y_{\vari \leftarrow \varsubsetj}  \label{eq:subip2} \\
& \label{eq:objcons} \sum_{\vari, \varsubsetj \ : \ 
\lpsol_{\vari \leftarrow \varsubsetj} > 0} \lpsol_{\vari \leftarrow \varsubsetj} \cdot
y_{\vari \leftarrow \varsubsetj} -\sum_{\vari \in \vertices} y_{\vari} > -1 \\
&  \ \ \ \ \quad \quad \quad \quad \quad \quad \quad \quad \quad   \label{eq:subint} y_{\vari \leftarrow \varsubsetj}, y_{\vari} \in \{0,1\}
\end{align}

The sub-IP constraints (\ref{eq:subip1}--\ref{eq:subip2}) are
displayed as propositional clauses for brevity, but note that these are linear
constraints. They can be written as $(1-y_{\vari \leftarrow
  \varsubsetj}) + y_{\vari} \geq 1$ and $(1-y_{\vari \leftarrow
  \varsubsetj}) + \sum_{\varj \in \varsubsetj} y_{\varj} \geq 1$,
respectively. The constraint \eqref{eq:objcons} dictates that only
solutions with objective value strictly greater than -1 are
allowed. In the \gobnilp{} implementation this constraint is effected by
directly placing a lower bound on the objective rather than posting
the linear constraint \eqref{eq:objcons}, since the former is more
efficient.

It is not difficult to show---\citeA{Bartlett2015} provide the
detail---that any feasible solution to sub-IP
(\ref{eq:subobj}--\ref{eq:subint}) determines a cutting plane for
$\lpsol$ and that a proof of the sub-IP's infeasibility establishes
that there is no such cutting plane. Since \gobnilp{} spends much of its
time solving sub-IPs in the hunt for cluster constraint cutting
planes, the issue of whether there is a better approach is
important. Is it really a good idea to set up a sub-IP each time a
cutting plane is sought? Is there some algorithm (perhaps a
polynomial-time one) that can be directly implemented to provide a
faster search for cutting planes? In Section~\ref{sec:complexity}
we make progress towards answering these questions. We show that the
weak separation problem is \emph{NP-hard} and so (assuming
$\mathrm{P}\neq \mathrm{NP}$) there is no polynomial-time algorithm
for weak separation.

\section{Three Polytopes related to the BNSL IP}
\label{sec:polytopes}


As explained in Section~\ref{sec:gobnilp}, in the basic \gobnilp{}
algorithm one first (i)~uses only the convexity constraints, then
(ii)~ adds cluster constraints, and, if necessary, (iii)~ branches on
variables to solve the IP. These three  stages correspond to three different
\emph{polytopes} which will be defined and analyzed in
Sections~\ref{sec:digraphpoly}--\ref{sec:fvpoly}.
Before providing this analysis we first give essential background on
linear inequalities, polytopes and polyhedra
\shortcite{conforti14:integ-progr}. We follow the notation of a
\citeA{conforti14:integ-progr}, which is standard throughout the
mathematical programming literature: for $x,y \in                            
\reals^{n}$, (1) ``$x \leq y$'' 
 means that $x_{i} \leq y_{i}$ for all $i = 1, \dots, n$.
and (2) ``$xy$'' where $x,y \in \reals^{n}$ is the scalar or `dot'
product (i.e.\ $x^{T}y$).

\subsection{Linear inequalities, polytopes and polyhedra}

\begin{definition}
  A point $x \in \reals^{n}$ is a \emph{convex combination} of points
  in $S \subseteq \reals^{n}$ if there exists a finite set of points
  $x^{1}, \dots, x^{p} \in S$ and scalars $\lambda_{1}, \dots,
  \lambda_{p}$ such that
\[
x = \sum_{j=1}^{p} \lambda_{j}x^{j}, \hspace*{10mm} \sum_{j=1}^{p} \lambda_{j} =
1, \hspace*{10mm} \lambda_{1}, \dots, \lambda_{p} \geq 0.
\]
\end{definition}

\begin{definition}
  The \emph{convex hull} $\mathrm{conv}(S)$ of a set $S \subseteq
  \reals^n$ is the inclusion-wise minimal convex set containing
  $S$, i.e.,  $\mathrm{conv}(S) = \{ x \in \reals^{n} \mid \mbox{
    $x$ is a convex combination of points in $S$} \}$.
\end{definition}

\begin{definition}
  A subset $P$ of $\reals^n$ is a \emph{polyhedron} if there exists a
  positive integer $m$, an $m \times n$ matrix $A$, and a vector $b \in
  \reals^{m}$ such that
\[
P = \{ x \in \reals^{n} \mid Ax \leq b \}.
\]
\end{definition}

\begin{definition}
  A subset $Q$ of $\reals^n$ is a \emph{polytope} if $Q$ is the convex
  hull of a finite set of vectors in $\reals^n$.
\end{definition}

\begin{theorem}[Minkowski-Weyl Theorem for Polytopes]
  \label{thm:mw}
    A subset $Q$ of $\reals^n$ is a \emph{polytope} if and only if $Q$
    is a bounded polyhedron.
\end{theorem}

What the Minkowski-Weyl Theorem for Polytopes states is that a
polytope can either be described as the convex hull of a finite set of points
or as the set of feasible solutions to some linear program. It follows
that, for a given linear objective, an optimal point can be found by
solving the linear program. This is a superficially attractive
prospect since linear programs can be solved in polynomial time.

Unfortunately, for NP-hard problems (such as BNSL) it is impractical
to create, let alone solve, the linear program due to the size of $A$
and $b$. Fully characterising the inequalities $Ax \leq b$ is
also typically difficult. However, it is useful to identify at least some
of these inequalities. These inequalities define \emph{facets} of the
polytope. A facet is a special kind of \emph{face} defined as
follows.

\begin{definition}
  \label{def:face}
  A \emph{face} of a polyhedron $P \subseteq \reals^{n}$ is a set of the form
\[
F := P \cap \{x \in \reals^{n} \mid cx = \delta \},
\]
where $cx \leq \delta$ ($c \in \reals^{n}, \delta \in \reals$) is a
\emph{valid inequality} for $P$, i.e., all points in $P$ satisfy
it. We say the inequality $cx \leq \delta$ \emph{defines} the face. A
face is \emph{proper} if it is non-empty and properly contained in
$P$. An inclusion-wise maximal proper face of $P$ is called a
\emph{facet}.
\end{definition}

So, for example, a cube is a 3-dimensional polytope (it is also a
polyhedron) with 6 2-dimensional faces, 12 1-dimensional faces and 6
0-dimensional faces (the vertices). The 2-dimensional faces are facets
since each of them is proper and not contained in any other face. The
convex hull of the 7 points $(x=1,y=0)$, $(x=2,y=0)$, $(x=3,y=0)$,
$(x=2,y=1)$, $(x=3,y=1)$, $(x=2,y=2)$ and $(x=3,y=2)$, whose boundary
is represented in Figure~\ref{fig:cuts}, is 2-dimensional and has 4
1-dimensional facets (shown in Figure~\ref{fig:cuts}) and 4 0-dimensional
faces. Note that the `good' cut in the right-hand figure of
Figure~\ref{fig:cuts} is a facet-defining inequality.

Facets are important since they are given by the `strongest'
inequalities defining a polyhedron. The set of all facet-defining
inequalities of a polyhedron provides a minimal representation
$Ax \leq b$ of that polyhedron, so any cutting plane which is not
facet-defining is thus `redundant' (see
\citeA[p.141]{wolsey98:_integ_progr} for the formal definition of
redundancy). Practically, facet-defining inequalities are good
inequalities to add as cutting planes since they, and they alone, are
guaranteed not to be dominated by any other valid inequality and
also not by any linear combination of other valid
inequalities. Identifying facets is thus an important step in
improving the computational efficiency of an IP approach.

A face of an $n$-dimensional polytope is a facet if and only if it has
dimension $n-1$. (Note that the 6 facets of a 3-dimensional cube are
indeed 2-dimensional.) To prove that a face $F$ has dimension $n-1$ it
is enough to find $n$ \emph{affinely independent} points in
$F$. Affine independence is defined as follows
\shortcite{wolsey98:_integ_progr}.
\begin{definition}
  The points $x^{1}, \dots x^{k} \in \reals^{n}$ are \emph{affinely
    independent} if the $k-1$ directions $x^{2}-x^{1}, \dots,
  x^{k}-x^{1}$ are linearly independent, or alternatively the $k$
  vectors $(x^{1},1), \dots (x^{k},1) \in \reals^{n+1}$ are linearly
  independent. 
\end{definition} 
\noindent
Note that if $x^{1}, \dots x^{k} \in \reals^{n}$ are
  linearly independent they are also affinely independent.



Having provided these basic definitions we now move on to consider three polytopes of increasing complexity:
the \emph{digraph polytope} (Section~\ref{sec:digraphpoly}), the \emph{cluster
polytope} (Section~\ref{sec:clusterpoly}) and finally, our main object of
interest, the \emph{family variable polytope}
(Section~\ref{sec:fvpoly}). 

\subsection{The Digraph Polytope}
\label{sec:digraphpoly}

The digraph polytope is simply the convex hull of all digraphs
permitted by \ppsalone. Before providing a formal account of
this polytope  we define some notation. For a given set of
nodes $\vertices$ and permitted parent sets $\pps{\vari}$, recall from
Section~\ref{sec:palim} that the set of families is defined as 
\[
  \arcsetsfull(\vertices,\ppsalone)  :=  \{\vari 
\leftarrow \varsubsetj \mid \vari \in \vertices, \varsubsetj \in
\pps{\vari} \}.
\]
Furthermore, we  notate the set of families that remain once the
empty parent set for each vertex is removed by
\[
\arcsets(\vertices,\ppsalone)  :=  \arcsetsfull(\vertices,\ppsalone) \setminus \{\vari \leftarrow \emptyset \mid
\vari \in \vertices\}.
\]
In this and subsequent sections  $\arcsets(\vertices,\ppsalone)$ 
will serve as an index set. We will abbreviate $\arcsetsfull(\vertices,\ppsalone)$ and
$\arcsets(\vertices,\ppsalone)$ to $\arcsetsfull$ and $\arcsets$ unless
it is necessary or useful to identify the node set $\vertices$ and
permitted parent sets $\pps{\vari}$. 

For any edge set $\arcs \subseteq \vertices \times
\vertices$, it is clear that any 0-1 vector in $\reals^{\arcs}$
corresponds to a (possibly cyclic) subgraph of $\digraphname =
\digraph$. However, there are many 0-1 vectors in $\reals^{\arcsetsfull}$
(or $\reals^{\arcsets}$) which do not
correspond to digraphs, namely those where $x_{\vari \leftarrow
  \varsubsetj} = x_{\vari \leftarrow \varsubsetj'} = 1$ for some
$\vari \leftarrow \varsubsetj, \vari \leftarrow \varsubsetj' \in
\arcsetsfull$ with $\varsubsetj \neq \varsubsetj'$. So clearly
inequalities other than simple variable bounds are required to
define the digraph polytope.

Since any digraph (cyclic or acyclic) satisfies the $|\vertices|$
convexity constraints \eqref{eq:convexitya}, the digraph polytope if
expressed using the variables in $\arcsetsfull$ will not be
\emph{full-dimensional}---the dimension of the polytope will be less than
the number of variables. This is inconvenient since only
full-dimensional polytopes have a unique minimal description in terms
of their facets.

To arrive at a full-dimensional polytope we remove the $|\vertices|$
family variables with empty parent sets and define the digraph
polytope using index set $\arcsets(\vertices,\ppsalone)$. Let
$\assppoly{\vertices,\ppsalone}$ be the \emph{digraph polytope} which
is the convex hull of all points in $\reals^{\arcsets(\vertices,\ppsalone)}$ that
correspond to digraphs (cyclic and acyclic).
\begin{align}
  \label{eq:psac}
  \assppoly{\vertices,\ppsalone} := &\mathrm{conv}\Big\{x \in
  \reals^{\arcsets(\vertices,\ppsalone)} \ \Bigm| \
\exists \arcss \subseteq \vertices  \times \vertices  \mbox{ s.t. } \\
& \ \ \ \ \ \ \ \ \ \   
  \ \paset{\vari}{\arcss} \in \pps{\vari} \ \forall \vari \in \vertices \mbox{ and } 
  x_{\vari\leftarrow\varsubsetj} =
  \indic{\varsubsetj = \paset{\vari}{\arcss}} \ \forall \varsubsetj \in \pps{\vari} \ \setminus \emptyset\Big\}.\nonumber
\end{align}
We will abbreviate $\assppoly{\vertices,\ppsalone}$ to
$\assppolyalone$ where this will not cause confusion.
\begin{proposition}
\label{prop:dim}
   $\assppolyalone$ is full-dimensional.
\end{proposition}
\begin{proof}
  The digraph with no edges is represented by the zero vector in
  $\reals^{\arcsets}$. Each vector in $\reals^{\arcsets}$ with only
  one component $x_{\vari \leftarrow \varsubsetj}$ set to 1 and all
  others set to 0 represents an acyclic digraph (denoted $e^{\vari
    \leftarrow \varsubsetj}$) and so is in $\assppolyalone$.
  These vectors together with the zero vector are clearly a set of
  $|\arcsets|+1$ affinely independent vectors from which it follows
  that $\assppolyalone$ is full-dimensional in
  $\reals^{\arcsets}$.
\end{proof}
$\assppolyalone$ is a simple polytope: it is easy to identify
all its facets.
\begin{proposition}
  \label{prop:asscvx}
  The facet-defining inequalities of $\assppolyalone$ are
  \begin{enumerate}
  \item  $\forall \vari \leftarrow \varsubsetj
    \in \arcsets: x_{\vari \leftarrow \varsubsetj} \geq 0$
    (variable lower bounds), and 
  \item $\forall \vari                                                
    \in \vertices :
    \sum_{\vari \leftarrow \varsubsetj \in \arcsets(\vertices,\ppsalone)}
    x_{\vari \leftarrow \varsubsetj} \leq 1$ (`modified' convexity constraints).
  \end{enumerate}
\end{proposition}
\begin{proof}
  We use Wolsey's third approach to establishing that a set of linear
  inequalities define a convex hull
  \cite[p.145]{wolsey98:_integ_progr}.  Let $c \in \reals^{\arcsets}$
  be an arbitrary objective coefficient vector.  It is clear that the
  linear program maximising $cx$ subject to the given linear
  inequalities has an optimal solution which is an integer vector
  representing a digraph: simply choose a `best' parent set for each
  $\vari \in \vertices$. (If all coefficients are non-positive choose
  the empty parent set.) Moreover for any digraph $x$, it easy to see
  that there is a $c$ such that $x$ is an optimal solution to the
  LP. It is also easy to see that each of the given linear
  inequalities is \emph{necessary}---removing any one of them results
  in a different polytope.  The result follows.
\end{proof}

Proposition~\ref{prop:asscvx} establishes the  unsurprising
fact that the polytope defined by \gobnilp's initial constraints is
$\assppoly{\vertices,\ppsalone}$, the convex hull of all digraphs
permitted by $\ppsalone$. It follows
that we will have $\lpsol \in \assppolyalone$ for any LP
solution $\lpsol$ produced by \gobnilp{} after adding cutting planes.

\subsection{The Cluster Polytope}
\label{sec:clusterpoly}

Although \gobnilp{} only adds those cluster constraints which are
needed to separate LP solutions $\lpsol$, it is useful to consider the
polytope which would be produced if all were added. The cluster
polytope $\ascppoly{\vertices,\ppsalone}$ is defined by adding all cluster
constraints to the facet-defining inequalities of the digraph
polytope $\assppoly{\vertices,\ppsalone}$, thus ruling out (family variable encodings of) cyclic
digraphs.
\begin{align}
  \label{eq:ascppoly}
\ascppoly{\vertices,\ppsalone} := &
\Big\{ x \in
  \reals^{\arcsets(\vertices,\ppsalone)} \ \Bigm|  \
x_{\vari \leftarrow \varsubsetj} \geq 0 \ \ \ \forall \vari \leftarrow
\varsubsetj \in \arcsets(\vertices,\ppsalone) \mbox{, and}\nonumber \\
    & \ \ \ \ \  \   
\sum_{\vari \leftarrow \varsubsetj \in \arcsets(\vertices,\ppsalone)}
    x_{\vari \leftarrow \varsubsetj} \leq 1 \ \ \ \  \forall \vari
    \mbox{, and} \nonumber \\
&  \ \ \ \ \  \ 
 \sum_{\vari\in\cluster}
\sum_{\varsubsetj \in \pps{\vari} : \varsubsetj
    \cap \cluster \neq \emptyset}
  x_{\vari \leftarrow \varsubsetj} \leq |\cluster| - 1 
\ \ \   \forall   \cluster \subseteq \vertices,\  |\cluster| > 1
\ \Big\}.\nonumber
\end{align}
We will abbreviate $\ascppoly{\vertices,\ppsalone}$ to
$\ascppolyalone$ where this will not cause confusion.
\begin{proposition}
   $\ascppolyalone$ is full-dimensional.
\end{proposition}
\begin{proof}
  Proof is essentially the same as that for Proposition~\ref{prop:dim}.
\end{proof}

As with the digraph polytope, we use the index set
$\arcsets$ to ensure full-dimensionality, and consequently
have to use formulation \eqref{eq:clusterineqalt} for cluster constraints.
Clearly $\ascppolyalone \subseteq \assppolyalone$ (and the
inclusion is proper if $|\vertices|>1$). Since
\gobnilp{} only adds some cluster constraints, the feasible set for each
LP that is solved during its cutting plane phase is a polytope
$P$ where $\ascppolyalone \subseteq P \subseteq
\assppolyalone$. More important is the connection between
$\ascppolyalone$ and the family variable polytope which we now
introduce.

\subsection{The Family Variable Polytope}
\label{sec:fvpoly}

The \emph{family variable polytope} $\fvpolyy{\vertices,\ppsalone}$ is
the convex hull of acyclic digraphs with node set $\vertices$ which
are permitted by \ppsalone. To define $\fvpolyy{\vertices,\ppsalone}$
it is first useful to introduce notation for the set of acyclic
subgraphs of some digraph. Let $\digraphname = \digraph$ be a digraph, and
\begin{equation}
  \label{eq:acyclic}
  {\cal \arcs}(\digraphname) := \{ \arcss \subseteq \arcs \mid \mbox{$\arcss$ is acyclic in $\digraphname$} \}.
\end{equation}
Now consider the case where $\digraphname = (\vertices,\vertices
\times \vertices)$. The
\emph{family variable polytope} $\fvpoly{\vertices,\ppsalone}$ is
\begin{align}
  \label{eq:pasac}
  \fvpolyy{\vertices,\ppsalone} := \mathrm{conv}\Big\{x \in
  \reals^{\arcsets(\vertices,\ppsalone)} \ \Bigm| \ &  \exists \arcss \in {\cal
    \arcs}(\digraphname)
\mbox{ s.t. } \paset{\vari}{\arcss} \in \pps{\vari} \  \forall \vari \in \vertices \mbox{ and}\\
& \ \ \ \ \ \ \ \ \ \ \ \ \quad \ \ \ \ \ \ \ x_{\vari\leftarrow\varsubsetj} =
  \indic{\varsubsetj = \paset{\vari}{\arcss}} \ \forall \varsubsetj \in \pps{\vari} \setminus \emptyset \Big\}.\nonumber
\end{align}
We will abbreviate $\fvpolyy{\vertices,\ppsalone}$ to
$\fvpolyyalone$ where this will not cause confusion.
\begin{proposition}
   $\fvpolyyalone$ is full-dimensional.
\end{proposition}
\begin{proof}
  Proof is essentially the same as that for Proposition~\ref{prop:dim}.
\end{proof}
It is clear that $\fvpolyyalone \subseteq \ascppolyalone
\subseteq \assppolyalone$. We will see in
Section~\ref{sec:facets} that although cluster constraints turn out to
be facet-defining inequalities of $\fvpolyyalone$, they are not the
only facet-defining inequalities, and so
(if $|\vertices|>2$) $\fvpolyyalone \subsetneq
\ascppolyalone$. We do, however, have that $\integers^{|\arcsets|}
\cap \fvpolyyalone
= \integers^{|\arcsets|} \cap \ascppolyalone$, since
acyclic digraphs are the only zero-one vectors to satisfy all cluster
and modified convexity
constraints. These facts have important consequences for the IP
approach to BNSL. They show that (i) cluster constraints are a good
way of ruling out cycles (since they are facet-defining inequalities of
$\fvpolyyalone$) and that (ii) one can solve a BNSL by just using
cluster constraints and branching on variables (to enforce an integral
solution). That $\fvpolyyalone \subsetneq
\ascppolyalone$ also implies that it may be worth searching for
facet-defining cuts which are not cluster inequalities, for example those discovered
by \citeA{studeny15:_how_bayes}.

\section{Computational Complexity of the BNSL Sub-IPs}
\label{sec:complexity}

In this section we focus on the computational complexity of the BNSL sub-IPs,
formalized as the weak separation problem for BNSL. As the main result of this section,
we show that this problem is NP-hard.

The \emph{weak separation problem} for BNSL is as follows: given
a $\lpsol \in \assppolyalone$, find a \emph{separating cluster} $\cluster
\subseteq \vertices$, $\card{\cluster} > 1$, for which
\begin{equation}
  \label{eq:violated}
 \sum_{\vari\in\cluster} \
\sum_{\varsubsetj \in \pps{\vari}: \varsubsetj
    \cap \cluster \neq \emptyset}
  \lpsol_{\vari \leftarrow \varsubsetj} > |\cluster| - 1,
\end{equation}
or establish that no such $\cluster$ exists.
We first give a simple necessary condition on separating clusters.

\begin{definition}
  Given $\lpsol \in \assppolyalone$ define $\lceil D
  \rceil(\lpsol)$, the \emph{rounding-up digraph} for $\lpsol$, as
  follows: $\vari \leftarrow \varj$ is an edge in $\lceil D
  \rceil(\lpsol)$ iff there is a family $\vari \leftarrow \varsubsetj$
  such that $\varj \in \varsubsetj$ and $\lpsol_{\vari \leftarrow 
    \varsubsetj} >0$.
\end{definition}

\begin{proposition}
  \label{prop:roundup}
  If $\cluster$ is a separating cluster for $\lpsol$, then $\lceil D
  \rceil(\lpsol)_{\cluster}$, the subgraph of the rounding-up digraph
  restricted to the nodes $\cluster$, is cyclic.
\end{proposition}

\begin{proof}
  Since $\lpsol \in \assppolyalone$, $\lpsol$ is a convex
  combination of extreme points of $\assppolyalone$. So we can
  write $\lpsol = \sum_{k=1}^{K} \alpha_{k}x^{k}$ where each $x^k$
  represents a graph and $\sum_{k=1}^{K} \alpha_{k} = 1$. For each
  graph $x^k$, let $x^{k}_{\cluster}$ be the subgraph restricted to
  the nodes $\cluster$. It is easy to see that if
  $x^{k}_{\cluster}$ is acyclic, then $\sum_{\vari\in\cluster}
  \sum_{\varsubsetj \in \pps{\vari}:\varsubsetj \cap \cluster \neq \emptyset}
  x^{k}_{\cluster_{\vari \leftarrow \varsubsetj}} \leq |\cluster| - 1 
  $. So if $x^{k}_{\cluster}$ is acyclic for all $k = 1,\dots,K$, then
  $\sum_{\vari\in\cluster} \sum_{\varsubsetj \in \pps{\vari}:\varsubsetj \cap \cluster \neq
    \emptyset} \lpsol_{\cluster_{\vari \leftarrow \varsubsetj}} \leq
  |\cluster| - 1 $. But if $\lceil D \rceil(\lpsol)_{\cluster}$ is
  acyclic, then so are all the $x^{k}_{\cluster}$. The result follows.
\end{proof}

Proposition~\ref{prop:roundup} leads to a heuristic algorithm for the
weak separation problem (which is available as an option in
\gobnilp). Given an LP solution $\lpsol$, the rounding up digraph
$\lceil D \rceil(\lpsol)$ is constructed and cycles in that digraph
are searched for using standard techniques. For each cycle found, the
corresponding cluster is checked to see whether it is a separating
cluster for $\lpsol$. We now consider the central result on weak separation.

\begin{theorem}\label{thm:subip-hardness}
The weak separation problem for BNSL is NP-hard, even when restricted
to instances $(V,\ppsalone, c)$ where 
$\varsubsetj \in \pps{\vari}$ for all $\vari \in \vertices$ only if $|\varsubsetj| \leq 2$. 
\end{theorem}

\begin{proof}
We prove the claim by reduction from vertex cover; that is, given a
graph $G = (V,E)$ and an integer $k$, we construct $\lpsol \in
\assppoly{V',\ppsdashalone}$ over a vertex set $V'$ and permitted
parent sets \ppsdashalone{} such that there is a cluster $C \subseteq V'$ with $|C| > 1$ and
\[ \sum_{\vari\in\cluster}  \ \sum_{\varsubsetj \in \ppsdash{\vari}:\varsubsetj
    \cap \cluster \neq \emptyset}
  \lpsol_{ \vari \leftarrow \varsubsetj} > |\cluster| - 1 \]
if and only if there is a vertex cover of size at most $k$ for $G$.

Specifically, let us denote $n = |V|$ and $m = |E|$. We construct $\lpsol \in \assppoly{V',\ppsdashalone}$ as follows;  Figure~\ref{fig:subip-hardness} illustrates the basic idea.
\begin{enumerate}
    \item The vertex set is $V' = V \cup S$, where $S$ is disjoint from $V$ and $|S| = m$.
    \item For $s \in S$ and $\{ u, v \} \in E$, we set $\lpsol_{s
        \leftarrow \{ u,v \}} = 1/m$; in particular, $\sum_{ \{u,v\} \in E} \lpsol_{s \leftarrow \{ u,v \}} = 1$ for all $s \in S$.
    \item For $s \in S$ and $v \in V$, we set
    \[ \lpsol_{v \leftarrow \{ s \}} = \frac{k}{m(k+1)}\,. \]
    \item $\lpsol_{i \leftarrow \emptyset} = 0$ for all   $i \in V'$. 
    \item For all other choices of $i \in V'$ and $J \subseteq V'
      \setminus \{i\}$: $J \not \in \ppsdash{i}$.
\end{enumerate}
Finally, for a cluster $C \subseteq V'$, we define the score $w(C)$ as
\[ w(C) = \sum_{\vari\in C }  \sum_{\varsubsetj \in \ppsdash{\vari}:\varsubsetj
    \cap C \neq \emptyset} \lpsol_{\vari \leftarrow \varsubsetj} - |C|\,. \]
Now we claim that there is a set $C \subseteq V'$ with $w(C) > - 1$ if and only if $G$ has a vertex cover of size at most $k$; this suffices to prove the claim.

First, we observe that if $U \subseteq V$ is a vertex cover in $G$, then
\begin{equation*}
\begin{aligned}
    w(U \cup S)  & = - |U| + \sum_{v \in U}\sum_{s \in S}x_{v \leftarrow \{ s \}} - |S| + \sum_{s \in S}\sum_{e \in E}\frac{1}{m} \\
                 & = - |U| + |U|\frac{m k}{m(k+1)} - |S| + |S|\frac{m}{m} \\
                 & = - |U|\Bigl( 1 - \frac{k}{k+1} \Bigr) = - \frac{|U|}{k+1}\,,
\end{aligned}
\end{equation*}
which implies that $w(U \cup S) > -1$ if $|U| \le k$.

Now let $C \subseteq V'$, and let us denote $C_{V} = C \cap V$ and $C_S = C \cap S$. If $|C_{V}| \ge k + 1$, then we have
\begin{equation*}
\begin{aligned}
    w(C)  & \le - |C_{V}| + |C_{V}|\frac{|C_S|k}{m(k+1)} \\
          & \le - |C_{V}| + |C_{V}|\frac{k}{k+1} \\
          & = - |C_{V}|\Bigl( 1- \frac{k}{k+1} \Bigr) = \frac{|C_{V}|}{k+1} \le - 1\,.
\end{aligned}
\end{equation*}

\begin{figure}
  \begin{center}
  \includegraphics[width=\linewidth]{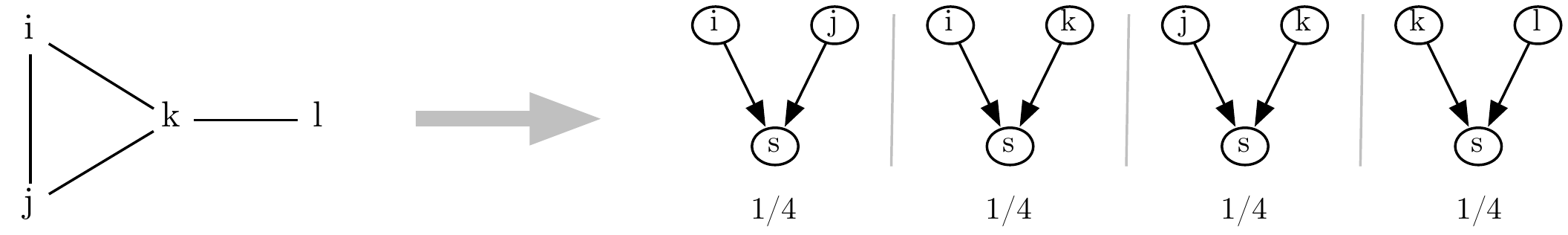}
  \end{center}
  \caption{The basic gadget of the reduction in Theorem~\ref{thm:subip-hardness}. Each edge ${u,v} \in V$ in the original instance $G =(V,E)$ is represented by assigning weight $\lpsol_{s \leftarrow \{ u,v \}} = 1/\card{E}$ in the new instance, where $s$ is a new node. Clearly, $U \subseteq V$ is a vertex cover in $G$ if and only if total weight of terms $\lpsol_{s \leftarrow \{ u,v \}}$ such that $U$ intersects the parent set is $1$.}
  \label{fig:subip-hardness} 
\end{figure}

On the other hand, let us consider the case where $|C_{V}| \le k$ but $C_{V}$ is not a vertex cover for $G$; we may assume that $C_{V} \not= \emptyset$, as otherwise we would have $w(C) = - |C| \le -1$. Let us write $H = \{ e \in E \mid C_{V} \cap e \not= \emptyset \}$ for the set of edges covered by $C_V$. Since we assume that $C_V$ is not a vertex cover, we have $|H| \le m-1$. Thus, it holds that
\begin{equation*}
\begin{aligned}
    w(C)  & = - |C_{V}| + |C_{V}|\frac{|C_S|k}{m(k+1)} - |C_S| + |C_S|\frac{|H|}{m} \\
          & \le - |C_{V}| + |C_{V}|\frac{|C_S| k}{m(k+1)} - |C_S| + |C_S|\frac{m-1}{m} \\
          & = - |C_{V}|\Bigl( 1 - \frac{|C_S| k}{m(k+1)}\Bigr) - \frac{|C_S|}{m} \\
          & \le - \Bigl( 1 - \frac{|C_S| k}{m(k+1)}\Bigr) - \frac{|C_S|}{m} \\
          & = - 1 - |C_S|\Bigl(\frac{1}{m} - \frac{k}{m(k+1)}\Bigr) = - 1 - \frac{|C_S|}{m(k+1)} < -1\,.
\end{aligned}
\end{equation*}
Thus, if $C_{V}$ is not a vertex cover of size at most $k$, then $w(C) \le -1$.

\end{proof}

\section{Facets of the Family Variable Polytope}
\label{sec:facets}

In this section a number of facets of the family variable polytope are
identified and certain properties of facets are
given. Section~\ref{sec:simple} provides simpler results, and
Sections~\ref{sec:clustersarefacets}--\ref{sec:restricting} 
more substantial ones, including a tight connection between facets and
cluster constraints, liftings of facets, and the influence of
restricting parent sets on facets. In Appendix~\ref{sec:lowdim} we provide a
complete enumeration of the facet-defining inequalities over
2--4 nodes and confirm the enumeration is consistent with the theoretical results presented
here.

\subsection{Simple Results on Facets}
\label{sec:simple}

We start by showing that the full-dimensional family variable polytope
$\fvpolyyalone$ is \emph{monotone} via a series of lemmas. Once
we have proved this result, we will use it to establish elementary
properties of facets of $\fvpolyyalone$ and find the simple
facets of the polytope.

\begin{definition}
  A nonempty polyhedron $P \subseteq \reals^{n}_{\geq 0}$ is \emph{monotone}
if $x \in P$ and $0 \leq y \leq x$ imply $y \in P$.
\end{definition}
\begin{lemma}
  \label{lem:zerocomp}
  Let $x \in \fvpolyyalone$ and let the vector $y$ be such
  that $y_{\vari' \leftarrow \varsubsetj'} = 0$ for some $\vari'
  \leftarrow \varsubsetj'$ and $y_{\vari \leftarrow \varsubsetj} =
  x_{\vari \leftarrow \varsubsetj}$ if $\vari \leftarrow \varsubsetj
  \neq \vari' \leftarrow \varsubsetj'$. Then $y \in
  \fvpolyyalone$.
\end{lemma}
\begin{proof}
  Since $x \in \fvpolyyalone$, $x = \sum \alpha_{k}x^{k}$ where
  each $x^k$ is an extreme point of $\fvpolyyalone$ corresponding
  to an acyclic digraph. For each $x^k$ define the vector $y^k$ where
  $y^{k}_{\vari' \leftarrow \varsubsetj'} = 0$ and all other
  components of $y^k$ are equal to those of $x^k$.  Each $y^k$ is also an
  extreme point corresponding to an acyclic digraph (a subgraph of
  $x^k$). We clearly have that $y = \sum \alpha_{k}y^{k}$ and so $y
  \in \fvpolyyalone$.
\end{proof}

\begin{lemma}
  \label{lem:fraccomp}
  Let $x \in \fvpolyyalone$ and let $y$ be any vector such
  that $0 \leq y_{\vari' \leftarrow \varsubsetj'} \leq x_{\vari' \leftarrow \varsubsetj'}$ for some $\vari'
  \leftarrow \varsubsetj'$ and $y_{\vari \leftarrow \varsubsetj} =
  x_{\vari \leftarrow \varsubsetj}$ if $\vari \leftarrow \varsubsetj
  \neq \vari' \leftarrow \varsubsetj'$. Then $y \in
  \fvpolyyalone$.
\end{lemma}
\begin{proof}
  If $x_{\vari' \leftarrow \varsubsetj'}=0$ then $y=x$ and the result
  is immediate, so assume that $x_{\vari' \leftarrow \varsubsetj'}>0$.
  Consider $z$ which is identical to $y$ except that
  $z_{\vari' \leftarrow \varsubsetj'}=0$.  We have
  $y = \frac{y_{\vari' \leftarrow \varsubsetj'}}{x_{\vari' \leftarrow
      \varsubsetj'}}x + \left(1-\frac{y_{\vari' \leftarrow
      \varsubsetj'}}{x_{\vari' \leftarrow \varsubsetj'}}\right)z$.  By
  Lemma~\ref{lem:zerocomp} $z \in \fvpolyyalone$. Since $x$ is also in
  $\fvpolyyalone$ and $y$ is a convex combination of $x$ and $z$ it
  follows that $y \in \fvpolyyalone$.
\end{proof}

\begin{proposition}
  \label{prop:monotone}
  $\fvpolyy{\vertices}$ is monotone.
\end{proposition}
\begin{proof}
  Suppose $x \in \fvpolyyalone$ and $0 \leq y \leq x$. Construct a sequence of
  vectors $x = y^{0}, y^{1}, \dots,$ $y^{k}, \dots, y^{|\arcsets|} = y$ by replacing
  each component $x_{\vari \leftarrow \varsubsetj}$ by $y_{\vari
    \leftarrow \varsubsetj}$ one at a time (in any order). By
  Lemma~\ref{lem:fraccomp} each $y^{k} \in \fvpolyyalone$,
  so $y \in \fvpolyyalone$.
\end{proof}

\shortciteA{hammer75:_facet} showed that a polytope is
monotone if and only if it can be described by a system $x\geq 0$, $Ax
\leq b$ with $A,b \geq 0$. This gives the following result for
$\fvpolyyalone$.

\begin{theorem}
  \label{thm:monofacets}
  Each facet-defining inequality of $\fvpolyy{\vertices}$ is either (i)~a lower bound (of
  zero) on a family variable,  or (ii)~an inequality of the form $\pi x \leq
  \pi_{0}$, where $\pi \geq 0$ and $\pi_{0} > 0$.
\end{theorem}
\begin{proof}
  From Proposition~\ref{prop:monotone} and the result of \shortciteA{hammer75:_facet} we have the result but with $\pi_{0} \geq 0$. That
  $\pi_{0} > 0$ follows directly by full-dimensionality.
\end{proof}

\begin{proposition} \label{prop:mc}
The following hold.
  \begin{enumerate}
  \item \label{lb} $x_{\vari \leftarrow \varsubsetj} \geq 0$ defines a facet of
    $\fvpolyy{\vertices,\ppsalone}$ for all families $\vari \leftarrow \varsubsetj \in
    \arcsets(\vertices,\ppsalone)$.
  \item \label{ub} For all $\vari \in \vertices$, if 
$\varsubsetj' \in \pps{\vari'}$ implies $\exists
    \varsubsetj \neq \emptyset \in \pps{\vari}$ for all other $\vari' \in \vertices$,
where $\vari \not\in \varsubsetj'$ or
    $\vari' \not \in \varsubsetj$, then
    $\sum_{\varsubsetj \neq \emptyset, \varsubsetj \in \pps{\vari}}
    x_{\vari \leftarrow \varsubsetj} \leq 1$ defines a facet of
    $\fvpolyy{\vertices,\ppsalone}$.
  \end{enumerate}
\end{proposition}
\begin{proof}
  (\ref{lb}) follows from the monotonicity of $\fvpolyy{\vertices,\ppsalone}$
  \shortcite[Proposition~2]{hammer75:_facet}. For (\ref{ub}) first
  define, for any $\vari \leftarrow \varsubsetj \in \arcsets(\vertices,\ppsalone)$ the
  unit vector $e^{\vari \leftarrow \varsubsetj} \in
  \reals^{\arcsets(\vertices,\ppsalone)}$, where $e^{\vari \leftarrow \varsubsetj}_{\vari
    \leftarrow \varsubsetj} = 1$ and all other components of $e^{\vari
    \leftarrow \varsubsetj}$ are 0.  For each $\vari \in \vertices$
  define $S_{\vari} = \{e^{\vari \leftarrow \varsubsetj} \mid \varsubsetj
  \neq \emptyset, \varsubsetj \in \pps{\vari}\} \cup \{e^{\vari' \leftarrow \varsubsetj'} + e^{\vari
    \leftarrow \varsubsetj} \mid \vari' \neq \vari, \varsubsetj' \neq
  \emptyset, \varsubsetj' \in \pps{\vari'}, \varsubsetj \neq
  \emptyset, \text{ and either $\vari \not\in
    \varsubsetj'$ or $\vari' \not\in
    \varsubsetj$} \}$.

There is an obvious
  bijection between family variables and the elements of $S_i$ so
  $|S_{\vari}| = |\arcsets(\vertices,\ppsalone)|$. It is easy to see that the vectors in
  $S_i$ are linearly independent (and thus affinely independent) and that each is an acyclic digraph
  satisfying 
      $\sum_{\varsubsetj \neq \emptyset, \varsubsetj \in \pps{\vari}}
    x_{\vari \leftarrow \varsubsetj} = 1$. The result
  follows.
\end{proof}

Recall that we use the name \emph{modified convexity constraints} to
describe inequalities of the form $\sum_{\varsubsetj \neq \emptyset, \varsubsetj \in \pps{\vari}} x_{\vari \leftarrow \varsubsetj} \leq 1$.  That each
node can have exactly one parent set in any digraph is a convexity
constraint. If we remove the empty parent set, this convexity constraint becomes an inequality, and
is thus \emph{modified}. We have now shown that each modified
convexity constraint defines a facet of $\fvpolyy{\vertices,\ppsalone}$ as long as a
weak condition is met. In fact, we have found this weak condition to be essentially always met
in practice. Note also that it is always met when all parent sets are allowed (as long as $|\vertices|>2$). 

We now show that  if $\coeffs x \leq \rhs$ defines a facet of the
family-variable polytope, then, for each
family, there is an acyclic digraph `containing' that family for which
 $\coeffs x \leq \rhs$ is `tight'. 

\begin{proposition}
  \label{prop:tighthness}
  If
  $\coeffs x \leq \rhs$ defines a facet of $\fvpolyyalone$ which is
  not a lower bound on a family variable, then for all families $\vari \leftarrow
  \varsubsetj \in \arcsets$, there exists an extreme point $x$ of
  $\fvpolyyalone$ such that $x_{\vari \leftarrow \varsubsetj} =
  1$ and $\coeffs x = \rhs$.
\end{proposition}
\begin{proof} Recall that by definition each extreme point of
  $\fvpolyyalone$ is a zero-one vector (representing an
  acyclic digraph). Now suppose that there were some $\vari \leftarrow \varsubsetj \in \arcsets$ such
  that $x_{\vari \leftarrow \varsubsetj} = 0$ for any extreme point
  $x$ of
  $\fvpolyyalone$ such that $\coeffs x =
  \rhs$. Since $\coeffs x \leq \rhs$ defines a facet, there is a set of
  $|\arcsets|$ affinely independent extreme points satisfying $\coeffs x =
  \rhs$. By our assumption, each such extreme point will also satisfy $x_{\vari
    \leftarrow \varsubsetj} =
  0$. $x_{\vari \leftarrow \varsubsetj} \geq 0$ defines a facet. However, it is not
  possible for a set of $|\arcsets|$ affinely independent points to
  lie on two distinct facets. The result follows.
\end{proof}

Proposition~\ref{prop:tighthness} helps us prove an important
property of facet-defining inequalities of $\fvpolyyalone$: coefficients are
non-decreasing as parent sets increase. The proof of the following
proposition rests on the simple fact that removing edges from an acyclic
digraph always results in another acyclic digraph.

\begin{proposition}
  \label{prop:nondecrease}
  Let $\coeffs x \leq \rhs$ be a facet-defining inequality of $\fvpolyyalone$.
     Then $\varsubsetj \subseteq \varsubsetj'$ implies
     $\coeffs_{\vari \leftarrow \varsubsetj} \leq  \coeffs_{\vari
       \leftarrow \varsubsetj'}$.
\end{proposition}
\begin{proof}
  Since $\coeffs x \leq \rhs$ defines a facet, there exists an extreme point
  $x'$ such that $x'_{\vari \leftarrow \varsubsetj'}=1$ and $\coeffs x' =
  \rhs$. Note that $x'_{\vari \leftarrow \varsubsetj}=0$.
Since $x'$ is an extreme point, it encodes an acyclic digraph.
Let $x$ be identical to $x'$ except that
  $x_{\vari \leftarrow \varsubsetj}=1$ and $x_{\vari \leftarrow \varsubsetj'}=0$. Since
  $\varsubsetj \subseteq \varsubsetj'$, $x$ also encodes an acyclic digraph
  and so is in $\fvpolyyalone$ so
  $\coeffs x \leq \rhs$. Thus $\coeffs x - \coeffs x' \leq 0$. However,
  $\coeffs x - \coeffs x' = \coeffs_{\vari \leftarrow \varsubsetj} -
  \coeffs_{\vari \leftarrow \varsubsetj'}$, and the result follows.
\end{proof}

\subsection{Cluster Constraints are Facets of the Family Variable Polytope}
\label{sec:clustersarefacets}

In this section we show that each \emph{$\palim$-cluster inequality}
is facet-defining for the family variable polytope in the special case where
the cluster $\cluster$ is the entire node set $\vertices$ and where
all parent sets are allowed for each vertex. The $\palim$-cluster
inequalities \cite{cussens11:_bayes_networ_learn_cuttin_planes} are a
generalisation of cluster inequalities (\ref{eq:clusterineq}). The
cluster inequalities (\ref{eq:clusterineq}) are $\palim$-cluster
inequalities for the special case of $\palim=1$.

In the next section (Section~\ref{sec:lifting}) we will show how to
`lift' facet-defining inequalities. This provides an easy
generalisation (Theorem~\ref{thm:kcluster}) of the result of this
section which shows that, when all parent sets are allowed, \emph{all} $\palim$-cluster inequalities are
facets, not just those for which $\cluster = \vertices$. 
As a special case, this implies that the cluster inequalities devised by
\shortciteA{jaakkola10:_learn_bayes_networ_struc_lp_relax} are facets
of the family variable polytope when all parent sets are allowed.

An alternative proof for the fact that
$\palim$-cluster inequalities are facet-defining 
was recently provided by \shortciteA[Corollary~4]{james15:_polyh_bayes}
The proof establishes not only that
$\palim$-cluster inequalities are facet-defining, but also that they are
\emph{score-equivalent}. A face of the family variable polytope
is said to be score-equivalent if it is the optimal face for some
\emph{score equivalent objective}, where the \emph{optimal face} of an
objective is the face containing all optimal solutions.
An objective function is score
equivalent if it gives the same value to any two acyclic digraphs which
are Markov equivalent (encode the same conditional independence
relations). In later work, \citeA{studeny15:_how_bayes}
went further and showed that $\palim$-cluster inequalities form just
part of a more general class of facet-defining inequalities which can
be defined in terms of \emph{connected matroids}.
However, we believe that our proof, as presented in the following, is valuable
since it relies only on a direct application of a standard
technique for proving that an inequality is facet-defining, and does not require
any connection to be made to score-equivalence, let alone matroid
theory. In addition, the general result (our
Theorem~\ref{thm:kcluster}) further shows how our results on `lifting' can be
usefully applied.

First we define \emph{$\palim$-cluster inequalities}.  There is a
$\palim$-cluster inequality for each cluster $\cluster \subseteq
\vertices$, $|\cluster|>1$, and each $\palim < |\cluster|$ which states
that there can be at most $|\cluster|-\palim$ nodes in $\cluster$
with at least $\palim$ parents in $\cluster$. It is clear that such
inequalities are at least valid, since all acyclic digraphs clearly
satisfy them. We begin by considering the special case of $\cluster =
\vertices$ where the $\palim$-cluster inequality states that there can
be at most $|\vertices| - \palim$ nodes with at least $\palim$
parents. 
We first introduce some helpful notation.
\begin{definition}
  $\ppsvalone$ is defined as follows: $\ppsv{\vari} := 2^{\vertices
    \setminus \{\vari\}}$, for all $\vari \in \vertices$.
\end{definition}

We will now show that $\palim$-cluster inequalities are
facet-defining.

\begin{theorem}
  \label{thm:specialkcluster}
  For any positive integer $\palim < |\vertices|$, the
  following valid inequality defines a facet of the family variable polytope $\fvpolyy{\vertices,\ppsvalone}$:
\begin{equation}
  \label{eq:specialkcluster}
  \sum_{\vari\in\vertices} \sum_{ \varsubsetj \subseteq \vertices
  \setminus \{\vari\}, |\varsubsetj| \geq \palim} 
  x_{\vari \leftarrow \varsubsetj} \leq |\vertices| - \palim.
\end{equation}
\end{theorem}
\begin{proof}
  An indirect method of establishing affine independence is used. It is
  given, for example, by \citeA[p.144]{wolsey98:_integ_progr}.
  Let $x^{1}, \dots, x^{\ndags}$ be the set of all acyclic digraphs in $\fvpolyy{\vertices,\ppsvalone}$
  satisfying
\begin{equation}
  \label{eq:kclustereq}
  \sum_{\vari\in\vertices} \sum_{ \varsubsetj \subseteq \vertices
  \setminus \{\vari\}, |\varsubsetj| \geq \palim} 
  x_{\vari \leftarrow \varsubsetj} = |\vertices| - \palim.
\end{equation}
Suppose that all these points lie on some generic hyperplane $\mu x =
\mu_{0}$. Now consider the system of linear equations
\begin{equation}
  \label{eq:mu}
  \sum_{\vari \in \vertices}
  \sum_{ \varsubsetj \neq \emptyset, \varsubsetj \subseteq \vertices
  \setminus \{\vari\}}
  \mu_{\vari \leftarrow \varsubsetj}x_{\vari \leftarrow \varsubsetj}^{\dagindex}
= \mu_{0} \mbox{
    for $\dagindex = 1, \dots, \ndags$}.
\end{equation}
Note that
$ \mathrm{dim} \;
\fvpolyy{\vertices,\ppsvalone} = |\arcsets(\vertices,\ppsvalone)| = |\vertices|(2^{|\vertices|-1} -
1)$ and so there are the same number of 
$\mu_{\vari \leftarrow \varsubsetj}$ variables.  The system \eqref{eq:mu}, in
the $|\vertices|(2^{|\vertices|-1} - 1) + 1$ unknowns $(\mu,\mu_{0})$, is
now solved. This is done in three stages. First we show that  
$\mu_{\vari \leftarrow \varsubsetj}$ must be zero if $|\varsubsetj| < \palim$. Then we show that the
remaining $\mu_{\vari \leftarrow \varsubsetj}$ must all have the same value.
Finally, we show that this common value is 1 whenever $\mu_0$ is $|\vertices| - \palim$. 

To do this it is useful to consider acyclic tournaments on
$\vertices$. These are acyclic digraphs where there is a directed edge between
each pair of distinct nodes. It is easy to see that
\begin{enumerate}
\item for any $\palim < |\vertices|$, every acyclic tournament on
  $\vertices$ satisfies  \eqref{eq:kclustereq}, and that
\item for any $x_{\vari \leftarrow \varsubsetj}$ there is an acyclic tournament,
   where $x_{\vari \leftarrow \varsubsetj} = 1$.
\end{enumerate}

Let $x$ be an acyclic tournament on $\vertices$ with $x_{\vari
  \leftarrow \varsubsetj}=1$ for some $\vari \in \vertices$,
$|\varsubsetj| < \palim$, i.e., $\varsubsetj$ is the non-empty parent
set for $\vari$ in $x$. Now consider $x'$ which is identical to $x$
except that $\vari$ has no parents, so that $x - x' = e^{\vari
  \leftarrow \varsubsetj}$. Since $x$ is an acyclic tournament it
satisfies \eqref{eq:kclustereq}. But it is also easy to see that $x'$
satisfies \eqref{eq:kclustereq}, since no parent set of size at least
$\palim$ has been removed. So $\mu_{\vari \leftarrow \varsubsetj} =
\mu e^{\vari \leftarrow \varsubsetj} = \mu ( x - x') = \mu x - \mu x'
= \mu_{0} - \mu_{0} = 0$. $\mu_{\vari \leftarrow
  \varsubsetj}=0$ whenever $|\varsubsetj| < \palim$. Call this Result~1.

Consider now two distinct parent sets $\varsubsetj$ and $\varsubsetj'$
for some $\vari \in \vertices$ where $\varsubsetj \geq \palim$ and
$\varsubsetj'\geq \palim$. Let $g$ be an acyclic tournament on the
node set $\vertices \setminus \{\vari\}$. Let $x$ be the acyclic
digraph on node set $\vertices$ obtained by adding $\{\vari\}$ to
$g$ and drawing edges from each member of $\varsubsetj$ to
$\vari$. Similarly, let $x'$ be the acyclic digraph obtained by drawing
edges from $\varsubsetj'$ to $\vari$ instead, so that $x - x' =
e^{\vari \leftarrow \varsubsetj} - e^{\vari \leftarrow
  \varsubsetj'}$. It is not difficult to see that both $x$ and $x'$
satisfy \eqref{eq:kclustereq}. So $\mu_{\vari \leftarrow \varsubsetj}
- \mu_{\vari \leftarrow \varsubsetj'} = \mu (e^{\vari \leftarrow
  \varsubsetj} - e^{\vari \leftarrow \varsubsetj'}) = \mu (x - x') =
\mu x - \mu x' = \mu_{0} - \mu_{0} = 0$.  So $\mu_{\vari \leftarrow
  \varsubsetj} = \mu_{\vari \leftarrow \varsubsetj'}$. Call this Result~2.

Now consider variables $x_{\vari \leftarrow \varsubsetj}$ and
$x_{\vari' \leftarrow \varsubsetj'}$ where $\vari \neq \vari'$,
$\varsubsetj \cup \{\vari\} = \varsubsetj' \cup \{\vari'\}$ and
$|\varsubsetj| = |\varsubsetj'| = \palim$.  First note that in an
acyclic tournament, (i)~there is exactly one parent set of each size
$0,\dots,\palim,\dots |\vertices|-1$ and so (ii)~the nodes of an
acyclic tournament can be totally ordered according to parent set size,
and thus (iii)~any total ordering of nodes  determines a unique
acyclic tournament.  Let $x$ be any acyclic tournament where $x_{\vari
  \leftarrow \varsubsetj} = 1$ and $x_{\vari' \leftarrow
  \varsubsetj^{(<\palim)}} = 1$ for some parent set $\varsubsetj^{(<\palim)}$
where $|\varsubsetj^{(<\palim)}| < \palim$.  Clearly there are many such
acyclic tournaments. Note that since $x$ is an acyclic tournament,
$\varsubsetj^{(<\palim)} \subseteq \varsubsetj \setminus \{\vari,\vari'\}$.
Now consider the acyclic tournament $x'$ produced by swapping $\vari$
and $\vari'$ in the total order associated with $x$. This generates an
acyclic tournament $x'$ where $x'_{\vari' \leftarrow \varsubsetj'} =
1$ and $x'_{\vari \leftarrow \varsubsetj^{(<\palim)}} = 1$. Note that
components of $x$ and $x'$ corresponding to family variables with
parent set size strictly above $\palim$ are equal. Components of
$\mu$ corresponding to family variables with parent set size strictly
below $\palim$ all equal zero.  From this we have that $\mu x - \mu x'
= \mu_{\vari \leftarrow \varsubsetj} - \mu_{\vari \leftarrow
  \varsubsetj'}$. Since $\mu x - \mu x' = \mu_{0} -
\mu_{0} = 0$, this shows that $\mu_{\vari \leftarrow \varsubsetj} =
\mu_{\vari' \leftarrow \varsubsetj'}$ Call this Result~3.

Now consider a pair of variables $\mu_{\vari \leftarrow \varsubsetj''}$ and
$\mu_{\vari' \leftarrow \varsubsetj'''}$ where $\vari \neq \vari'$, and
the only restriction is that $|\varsubsetj''|,|\varsubsetj'''| \geq
\palim$. If some other  pair of variables $\mu_{\vari \leftarrow \varsubsetj}$ and
$\mu_{\vari' \leftarrow \varsubsetj'}$ meet the conditions of
Result~3, then $\mu_{\vari \leftarrow \varsubsetj} =
\mu_{\vari' \leftarrow \varsubsetj'}$. However, by Result~2
$\mu_{\vari \leftarrow \varsubsetj''} = \mu_{\vari \leftarrow
  \varsubsetj}$
and $\mu_{\vari' \leftarrow \varsubsetj'''} = \mu_{\vari' \leftarrow
  \varsubsetj'}$. Thus $\mu_{\vari \leftarrow \varsubsetj''} =
\mu_{\vari' \leftarrow \varsubsetj'''}$.

So by the transitivity of equality $\mu_{\vari \leftarrow \varsubsetj} =
\mu_{\vari' \leftarrow \varsubsetj'}$ for any $\vari, \vari', \varsubsetj,
\varsubsetj'$ where $|\varsubsetj| \geq \palim$, $|\varsubsetj'| \geq
\palim$. Recall that we also have that $\mu_{\vari \leftarrow \varsubsetj} = 0 $
whenever $|\varsubsetj|<\palim$.

Suppose that $\mu_{0}=0$. Since all non-zero
$\mu_{\vari \leftarrow \varsubsetj}$ are equal and thus have the same sign, the
only possible solution is for all $\mu_{\vari \leftarrow \varsubsetj} = 
0$. Suppose then instead that $\mu_{0} \neq 0$. Then wlog we can
set $\mu_{0}=|\vertices| - \palim$. In each of the $\ndags$ equations
\eqref{eq:mu}, after substituting $\mu_{\vari \leftarrow \varsubsetj} = 0$ for
$|\varsubsetj|<\palim$, we have $|\vertices|-\palim$ terms on the left hand side (LHS)
which are known to be equal. On the right hand side (RHS) the value is
$|\vertices|-\palim$, so all terms on the LHS must equal one. Each term
$\mu_{\vari \leftarrow \varsubsetj}$ where $|\varsubsetj|\geq \palim$, occurs in
at least one of $\ndags$ equations \eqref{eq:mu}, so this is enough to
establish that $\mu_{\vari \leftarrow \varsubsetj} = 1$ whenever
$|\varsubsetj|\geq \palim$. Thus, unless all $\mu_{\vari \leftarrow \varsubsetj} =
0$, the only possible solution to the system of linear equations
(\ref{eq:mu}) with RHS $|\vertices| - \palim$ is
\begin{itemize}
\item $\mu_{\vari \leftarrow \varsubsetj} = 0$ if $|\varsubsetj|<\palim$, and
\item $\mu_{\vari \leftarrow \varsubsetj} = 1$ if $|\varsubsetj|\geq \palim$. 
\end{itemize}
These values match those in (\ref{eq:kclustereq}) and so
(\ref{eq:specialkcluster}) is facet-defining.
\end{proof}

\subsection{Lifting Facets of the Family Variable Polytope}
\label{sec:lifting}

In this section we show that if all parent sets are allowed, then facet-defining inequalities for the family variable polytope for some node
set $\vertices$ can be `lifted' to provide facets for any family
variable polytope for an enlarged node set $\vertices' \supsetneq
\vertices$.

\begin{lemma}
\label{lem:extend}
Recall that  $\ppsv{\vari} := 2^{\vertices \setminus \{\vari\}}$ for all $\vari \in \vertices$.
  Let
  \begin{equation}
    \label{eq:pfaceti}
    \sum_{\vari\in\vertices} \sum_{\varsubsetj \in \ppsv{\vari},
      \varsubsetj \neq \emptyset
      } 
    \alpha_{\vari \leftarrow \varsubsetj} x_{\vari \leftarrow \varsubsetj} \leq \beta    
  \end{equation}
  be a facet-defining inequality for the family variable polytope
  $\fvpolyy{\vertices,\ppsvalone}$ which is not a lower bound on a 
  variable.   Let $\vertices' = \vertices \cup \{\vari'\}$
  where $\vari' \not\in \vertices$. 
Then
  \begin{equation}
    \label{eq:p1faceti}
    \sum_{\vari\in\vertices} \sum_{\varsubsetj \in \ppsv{\vari},
      \varsubsetj \neq \emptyset} 
    \alpha_{\vari \leftarrow \varsubsetj} (
    x_{\vari \leftarrow
      \varsubsetj} + x_{\vari \leftarrow
      \varsubsetj \cup \{\vari'\}}) \leq \beta    
  \end{equation}
  is a facet-defining inequality of $\fvpolyy{\vertices',\ppsvdashalone}$. Furthermore,
  this inequality is not a
  lower bound on a variable.
\end{lemma}
\begin{proof}
  Since \eqref{eq:pfaceti} is  facet-defining, there is a set $S_0 \subseteq \reals^{\arcsets(\vertices,\ppsvalone)}$ of affinely
  independent acyclic digraphs, with node set $\vertices$, lying on
  its hyperplane. For each acyclic digraph in $S_0$, create an acyclic
  digraph with node set $\vertices \cup \{\vari'\}$ by adding
  $\vari'$ as an isolated node.  Let $S_{1} \subseteq \reals^{\arcsets(\vertices',\ppsvdashalone)}$ be the set of acyclic
  digraphs so created. Note that all members of $S_{1}$ lie on the
  hyperplane for \eqref{eq:p1faceti}. Each vector in $S_1$ corresponds
  to a vector in $S_0$ with a zero vector of length
  $|\arcsets(\vertices',\ppsvdashalone)|-|\arcsets(\vertices,\ppsvalone)|$ concatenated. Since $S_0$ is
  an affinely independent set, so is $S_1$. 
  
  For each non-empty subset $\varsubsetj \subseteq \vertices$, construct an acyclic digraph
  by adding  $e^{\vari' \leftarrow \varsubsetj}$ to an arbitrary member of
  $S_1$. Clearly the end result is an acyclic digraph lying on the hyperplane for
  (\ref{eq:p1faceti}). Let $S_2$ be the set of all such acyclic digraphs.

  For each $\varsubsetj \subseteq \vertices$, $\vari \in
  \vertices$, construct an acyclic digraph by finding an acyclic digraph $x \in S_{1}$ such that
  $x_{\vari \leftarrow \varsubsetj} = 1$ and adding an arrow from $\vari'$ to
  $\vari$.  Note that it is always possible to find an acyclic digraph with
  $x_{\vari \leftarrow \varsubsetj} = 1$. If this were not the case, then
  \eqref{eq:pfaceti} would be a lower bound on
  $x_{\vari \leftarrow \varsubsetj}$. It is not difficult to see that any such acyclic digraph
  lies on the hyperplane defined by \eqref{eq:p1faceti}. Let $S_3$ be the set of
  all such acyclic digraphs.

  Let $S= S_{1} \cup S_{2} \cup S_{3}$.  $S_{2}$ and $S_{3}$ have
  exactly one acyclic digraph for each component $x_{\vari \leftarrow \varsubsetj}$ involving
  the node $\vari'$ (either $\vari=\vari'$ or $\vari' \in
  \varsubsetj$). $S_{1}$ has an acyclic digraph for each component
  $x_{\vari \leftarrow \varsubsetj}$ not involving $\vari'$. So $|S| = \mathrm{dim} \;
  \fvpolyy{\arcsets(\vertices',\ppsvdashalone)} = |\arcsets(\vertices',\ppsvdashalone)|$. It remains to
  be established that the $S$ is a set of affinely independent vectors.

  Suppose $\sum_{x^{i} \in S}\alpha_{i}x^{i}=0$ and $\sum_{x^{i} \in
    S}\alpha_{i}=0$. Each component $x_{\vari \leftarrow \varsubsetj}$ involving
  $\vari'$ is set to 1 in exactly one acyclic digraph in $S_{2} \cup S_{3}$. Thus
$\alpha_{i}=0$ for $x^{i} \in S_{2} \cup S_{3}$. So
  $\sum_{x^{i} \in S_{1}}\alpha_{i}x^{i}=0$ and $\sum_{x^{i} \in
    S_{1}}\alpha_{i}=0$. The result then follows from the affine
  independence of the set $S_1$.
\end{proof}

\begin{theorem}
\label{thm:extend}
 Recall that $\ppsv{\vari} := 2^{\vertices \setminus \{\vari\}}$
  for all $\vari \in \vertices$.
  Let
  \begin{equation}
    \label{eq:pfacet}
    \sum_{\vari\in\vertices} \sum_{\varsubsetj \in \ppsv{\vari},
      \varsubsetj \neq \emptyset} 
    \alpha_{\vari \leftarrow \varsubsetj} x_{\vari \leftarrow \varsubsetj} \leq \beta    
  \end{equation}
  be a facet-defining inequality of the family variable polytope
  $\fvpolyy{\vertices,\ppsvalone}$ which is not a lower bound on a 
  variable.   Let $\vertices'$ be a node set such that $\vertices
  \subseteq \vertices'$. 
Then
  \begin{equation}
    \label{eq:p1facet}
    \sum_{\vari\in\vertices} \sum_{\varsubsetj \in \ppsv{\vari},
      \varsubsetj \neq \emptyset} 
    \alpha_{\vari \leftarrow \varsubsetj} \left(
      \sum_{\varsubsetj':\varsubsetj \subseteq \varsubsetj' \subseteq
        \vertices' \setminus \{\vari\}} x_{\vari
        \leftarrow \varsubsetj'}
    \right) \leq \beta    
  \end{equation}
  is facet-defining for $\fvpolyy{\vertices',\ppsvdashalone}$ and is not a
  lower bound on a variable.
\end{theorem}
\begin{proof}
  Repeated application of Lemma~\ref{lem:extend}.
\end{proof}

Using Theorem~\ref{thm:extend},
Theorem~\ref{thm:specialkcluster} can now be `lifted' to establish
that all $k$-cluster inequalities are facet-defining.

\begin{theorem}
  \label{thm:kcluster}
Recall that $\ppsv{\vari} := 2^{\vertices \setminus \{\vari\}}$ for all $\vari \in \vertices$. For any $\cluster \subseteq \vertices$ and any positive integer $\palim < |\cluster|$,
the  valid inequality 
\begin{equation}
  \label{eq:kcluster}
  \sum_{\vari\in\cluster} \ \ 
  \sum_{ \varsubsetj \subseteq \vertices \setminus \{\vari\} : |\varsubsetj \cap \cluster| \geq \palim} 
  x_{\vari \leftarrow \varsubsetj} \leq |\cluster| - \palim
\end{equation}
is facet-defining for the family variable polytope $\fvpolyy{\vertices,\ppsvalone}$.
\end{theorem}
\begin{proof}
  By Theorem~\ref{thm:specialkcluster}, \eqref{eq:kcluster} is 
  facet-defining for the family
  variable polytope for node set $\cluster$. By applying
  Theorem~\ref{thm:extend} it follows that it also facet-defining for the
family
  variable polytope for any node set $\vertices \supseteq \cluster$. 
\end{proof}

\subsection{Facets when Parent Sets are Restricted}
\label{sec:restricting}

The results in the preceding sections have all been for the special
case $\ppsvalone$ when all possible parent sets are allowed for each
node. If some parent sets are ruled out, for example by an upper
bound $\palim$ on parent set cardinality, then some $\palim$-cluster
inequalities and some modified convexity constraints may not be facets.

To see this, suppose we had $\vertices = \{a,b,c\}$. If all parent sets
are allowed, then Theorem~\ref{thm:kcluster} shows that this 2-cluster
inequality for $C=\{a,b,c\}$,
\begin{equation}
  \label{eq:twoabc}
  x_{a \leftarrow \{b,c\}} + x_{b \leftarrow \{a,c\}} + x_{c
    \leftarrow \{a,b\}} \leq 1,
\end{equation}
is facet-defining. However, if $\{a,b\}$ is not allowed as a parent set for $c$, then the inequality
becomes
\begin{equation}
  x_{a \leftarrow \{b,c\}} + x_{b \leftarrow \{a,c\}}  \leq 1,
\end{equation}
which is not facet-defining since it is dominated by the 1-cluster inequality
for $C=\{a,b\}$,
\begin{equation}
  \label{eq:oneab}
  x_{a \leftarrow \{b\}} + x_{a \leftarrow \{b,c\}} + x_{b \leftarrow
    \{a\}} + x_{b \leftarrow \{a,c\}}  \leq 1.
\end{equation}
As another example, suppose $\{c\}$ were removed from $\pps{a}$. Then
condition \ref{ub} of Proposition~\ref{prop:mc}  is no longer met,
and the modified convexity constraint for $a$ becomes
\begin{equation}
  \label{eq:mcc}
  x_{a \leftarrow \{b\}} + x_{a \leftarrow \{b,c\}} \leq 1,
\end{equation}
which cannot be facet-defining since it is dominated by the inequality \eqref{eq:oneab}.

For any $\ppsalone$ we have that the polytope
$\fvpolyy{\vertices,\ppsalone}$ is a \emph{face} of the
all-parent-sets-allowed polytope $\fvpolyy{\vertices,\ppsvalone}$
defined by the  valid inequality
\begin{equation}
  \label{eq:vface}
  \sum_{\vari \in \vertices}\sum_{\varsubsetj \in \ppsv{\vari}
    \setminus \pps{\vari}} 
  x_{\vari \leftarrow \varsubsetj} \geq 0.
\end{equation}

The issue then is whether it is possible to determine when a facet of
$\fvpolyy{\vertices,\ppsvalone}$ is also a facet of this face. The
issue of determining the facets of a face is of general interest. As
\shortciteA{Boyd2009} note
``As it is often technically much simpler to obtain results about
facets for a full dimensional polyhedron than one of lower
dimension, it would be nice to \dots 
know under what conditions an inequality inducing a facet of $P$
also induces a facet of a face $F$ of $P$ .''  They go on to
state that ``\dots we know of no reasonable general result of this type''.

However, in the case of the the family variable polytope, there is a
strong result which shows that many facets of a family variable
polytope $\fvpolyy{\vertices,\ppsalone}$ induce facets of a
lower-dimensional family variable polytope
$\fvpolyy{\vertices,\breve{\ppsalone}}$ where $\breve{\ppsalone}(\vari)
\subseteq \ppsalone(\vari)$ for all $\vari \in \vertices$. In
particular, this result shows that some facets of the all-parent-sets-allowed
polytope $\fvpolyy{\vertices,\ppsvalone}$ are also facets of the
polytope that results by limiting the cardinality of parent sets. To
establish this result we first prove a lemma.

\begin{lemma}
\label{lem:supermono}
Let $x \in \fvpolyy{\vertices,\ppsalone}$. Let $\vari \in \vertices$
and let $\varsubsetj, \varsubsetj' \in \pps{\vari}$ with
$\varsubsetj \subsetneq \varsubsetj'$, $\varsubsetj \neq \emptyset$. Define
$\breve{x}$ as
follows: $\breve{x}_{\vari \leftarrow
  \varsubsetj} = x_{\vari \leftarrow
  \varsubsetj} + x_{\vari \leftarrow
  \varsubsetj'}$, $\breve{x}_{\vari \leftarrow
  \varsubsetj'} =0$, and $x$ and
$\breve{x}$ are equal in all other
components. Then $\breve{x}$ is also in the
family-variable polytope $\fvpolyy{\vertices,\ppsalone}$.
\end{lemma}

\begin{proof}
  Since $x \in \fvpolyy{\vertices,\ppsalone}$, $x = \sum_{k=1}^{K}
  \alpha_{k}x^{k}$ where each $x^k$ is an extreme point of
  $\fvpolyy{\vertices,\ppsalone}$ corresponding to an acyclic
  digraph. For each $x^k$ define $\breve{x}^{k}$ as follows:
  $\breve{x}^{k}_{\vari \leftarrow \varsubsetj} = x^{k}_{\vari
    \leftarrow \varsubsetj} + x^{k}_{\vari \leftarrow \varsubsetj'}$,
  $\breve{x}^{k}_{\vari \leftarrow \varsubsetj'} =0$ and $x^k$ and
  $\breve{x}_{k}$ are equal in all other components. It is clear that
  each $\breve{x}^{k}$ corresponds to an acyclic digraph which differs
  from $x^{k}$ iff $\varsubsetj'$ is the parent set for $\vari$ in
  $x^{k}$, in which case $\varsubsetj$ becomes the parent set for
  $\vari$ in $\breve{x}^{k}$. The digraph remains acyclic since
  $\varsubsetj \subsetneq \varsubsetj'$. It is also clear that  $\breve{x} = \sum_{k=1}^{K}
  \alpha_{k}\breve{x}^{k}$ and so $\breve{x} \in  \fvpolyy{\vertices,\ppsalone}$.
\end{proof}

The main result of this section now follows. Our proof  makes use of the elementary but useful
fact that the number of linearly independent rows in a matrix (row
rank) and the number of linearly independent columns in a matrix
(column rank) are equal.

\begin{theorem}
\label{thm:facetrestrict}
  Let $\pi x \leq \pi_{0}$ define a facet for the family-variable
  polytope $\fvpolyy{\vertices,\ppsalone}$. Suppose that $\pi_{\vari
    \leftarrow \varsubsetj} = \pi_{\vari \leftarrow \varsubsetj'}$ for
  some $\vari \in \vertices$, $\varsubsetj, \varsubsetj' \in
  \pps{\vari}$ with $\varsubsetj \subsetneq \varsubsetj'$,
  $\varsubsetj \neq \emptyset$. Let
  $\breve{\pi}$ be $\pi$ with the component $\pi_{\vari \leftarrow
    \varsubsetj'}$ removed. Let $\breve{\ppsalone}$ be identical to
  $\ppsalone$ except that $\varsubsetj'$ is removed from
  $\pps{\vari}$. Then  $\breve{\pi} x \leq \pi_{0}$ defines a facet
  for the polytope $\fvpolyy{\vertices,\breve{\ppsalone}}$.
\end{theorem}

\begin{proof}

  Since $\pi x \leq \pi_{0}$ is facet-defining for
  $\fvpolyy{\vertices,\ppsalone}$ it is obvious by Theorem~\ref{thm:monofacets} that $\breve{\pi} x
  \leq \pi_{0}$ is at least a valid inequality for
  $\fvpolyy{\vertices,\breve{\ppsalone}}$. We now show that this valid
  inequality defines a facet by proving the existence of
  $|\arcsets(\vertices,\breve{\ppsalone})|$ affinely independent
  points lying in the facet.

  Recall that $\arcsets(\vertices,\ppsalone)$ is the set of families
  determined by vertices $\vertices$ and allowed parent sets
  $\ppsalone$. Abbreviate $|\arcsets(\vertices,\ppsalone)|$ to
  $m$ and note that $|\arcsets(\vertices,\breve{\ppsalone})| = m-1$. Since $\pi x \leq \pi_{0}$ defines a facet for the
  family-variable polytope $\fvpolyy{\vertices,\ppsalone}$, there are
  $m$ affinely independent points $x^{1}, \dots, x^{k}, \dots , x^{m}$
  lying in this facet (i.e., $\pi x^{k} = \pi_{0}$, $x^{k} \in \fvpolyy{\vertices,\ppsalone}$ for
  $k=1,\dots,m$). Since these points are affinely independent, the 
  points $(x^{1},1), \dots, (x^{k},1), \dots , (x^{m},1)$ in
  $\reals^{m+1}$ are linearly
  independent.

  Let $A_{1}$ be the $m \times (m+1)$ matrix whose rows are the
  $(x^{k},1)$. Since the rows are linearly independent, $A_1$ has rank
  $m$. Construct a new matrix $A_{2}$ by adding the column for family
  $\vari \leftarrow \varsubsetj'$ to that for $\vari \leftarrow
  \varsubsetj$. Since this is an elementary operation it does not
  change the rank of the matrix \shortcite{cohn}, and so $A_{2}$ has rank
  $m$.  Now construct an $m \times m$ matrix $A_3$ by removing the
  column for $\vari \leftarrow \varsubsetj'$ from $A_2$. Denote the
  rows of $A_3$ by $(\breve{x}^{1},1), \dots, (\breve{x}^{k},1), \dots
  , (\breve{x}^{m},1)$. From Lemma~\ref{lem:supermono} it follows that
  each $\breve{x}^{k}$ is in
  $\fvpolyy{\vertices,\breve{\ppsalone}}$. Since $\pi_{\vari
    \leftarrow \varsubsetj} = \pi_{\vari \leftarrow \varsubsetj'}$, it
  is not difficult to see that each $\breve{x}^{k}$ satisfies
  $\breve{\pi} x = \pi_{0}$. Since $A_2$ has rank $m$, there are $m$
  linearly independent columns in $A_2$ and, since $A_3$ is $A_2$ with
  one column removed, at least $m-1$ linearly independent columns in
  $A_3$. So $A_3$ has rank of at least $m-1$. But this means that
  there are $m-1$ linearly independent rows in $A_3$, so there are
  $m-1$ points among the $\breve{x}^{k}$ that are affinely
  independent. So there are $m-1$ affinely independent points in $\fvpolyy{\vertices,\breve{\ppsalone}}$
  satisfying $\breve{\pi} x = \pi_{0}$ and thus $\breve{\pi} x \leq
  \pi_{0}$ defines a facet of $\fvpolyy{\vertices,\breve{\ppsalone}}$.
\end{proof}

Given a facet-defining inequality of an all-parent-sets-allowed polytope
$\fvpolyy{\vertices,\ppsvalone}$ and a parent set cardinality limit
$\palim$, Theorem~\ref{thm:facetrestrict} states that if the
coefficients for all family variables $x_{\vari \leftarrow
  \varsubsetj'}$ with $|\varsubsetj'| > \palim$ are not strictly larger
than the coefficient for some family variable $x_{\vari \leftarrow
  \varsubsetj}$ with $\varsubsetj \subsetneq  \varsubsetj'$
so that $|\varsubsetj| \leq \palim$, then the inequality also defines a facet for
the polytope with family variables restricted by $\palim$. In
Appendix~\ref{sec:lowdim} this is confirmed for the case where
$|\vertices|=4$ and $\palim=2$.
It follows that a normal ($k=1$) cluster constraint is a
facet for \emph{any} limit $\palim$ on the size of parent sets. This
explains why normal cluster constraints are more useful to look for
than $k$-cluster constraints for $k>1$. In \gobnilp, although the user
can ask the system to look for $k$-cluster constraints up to some
defined limit $k \leq K$, the default is to only search for normal
($k=1$) cluster constraints since this has been observed to lead to
faster solving.

\section{Faces of the Family Variable Polytope defined by Orders and
  by Sinks}
\label{sec:faces}

In this section we analyse faces of the all-parent-sets-allowed family
variable polytope defined by total orders and sink nodes,
respectively. Faces of a polytope are themselves polytopes, and in
this section we establish a complete characterisation of the
facets of both types of polytope. Moreover, the faces defined by sink
nodes lead to a useful \emph{extended representation} for the family
variable polytope which can be used to relate family variable
polytopes for different numbers of nodes.

\subsection{Order-defined Faces}
\label{sec:orderfaces}

Let $<$ be some total order on the node set $\vertices$. An acyclic
digraph $(\vertices,\arcss)$ is  \emph{consistent with} $<$
if $\vari \leftarrow \varj \in \arcss \Rightarrow \varj < \vari$, so
that parents come before children in the ordering.  The
valid inequality $\sum_{\vari,  \varsubsetj : (\exists \varj \in
  \varsubsetj \mbox{ s.t. } \vari < \varj)}
x_{\vari
  \leftarrow \varsubsetj} \geq 0$ defines a face of the family
variable polytope 
\begin{equation}
  \label{eq:orderface}
  \fvpolyyo{\vertices} = \Bigl\{ x \in \fvpolyy{\vertices,\ppsvalone} \Bigm|
\sum_{\vari,  \varsubsetj:  (\exists \varj \in
  \varsubsetj: \vari < \varj)}
 x_{\vari
  \leftarrow \varsubsetj} = 0 \Bigr\}.
\end{equation}
In $\fvpolyyo{\vertices}$ each family variable inconsistent with $<$
is set to zero. This is the only restriction on $x$. So clearly all
acyclic digraphs consistent with $<$ lie on the face
$\fvpolyyo{\vertices}$ and no digraphs inconsistent with $<$ do. It is
also clear that any acyclic digraph lies on
$\fvpolyyo{\vertices}$ for at least one choice of $<$.
\begin{remark}
  Abbreviate $|\vertices|$ to $p$. We have that
  $\mathrm{dim}(\fvpolyyo{\vertices}) = 2^{p}-p-1$. If the family variables
  clamped to zero in $\fvpolyyo{\vertices}$ are removed,
  $\fvpolyyo{\vertices}$ is full-dimensional in $\reals^{2^{p}-p-1}$. (Recall that
  $\mathrm{dim}(\fvpolyy{\vertices,\ppsvalone}) = p(2^{p-1}-1)$.)
\end{remark}
\begin{remark}
  If $x$ is an extreme point of $\fvpolyy{\vertices,\ppsvalone}$, then $x \in \bigcup_{<}
      \fvpolyyo{\vertices}$.
\end{remark}
 Note that
exactly one acyclic tournament lies on $\fvpolyyo{\vertices}$ for any
choice of $<$.

\begin{proposition}
  \label{prop:order}
  The facet-defining inequalities of the full-dimensional polytope $\fvpolyyo{\vertices}$
  $\subseteq \reals^{2^{p}-p-1}$ are
  \begin{enumerate}
  \item the variable lower bounds $x_{\vari \leftarrow \varsubsetj}
    \geq 0$, and 
  \item the modified convexity constraints $\sum_{J \subseteq
      \vertices : J \neq \emptyset, \varj \in J \rightarrow j < i} x_{\vari \leftarrow \varsubsetj} \leq
    1$,
  \end{enumerate}
  where variables $x_{\vari \leftarrow \varsubsetj}$ with 
  $\varj \in \varsubsetj, i < j$ have been removed.
\end{proposition}
\begin{proof}
  Let $c \in \reals^{2^{p}-p-1}$ be an arbitrary objective coefficient
  vector. Consider solving the LP with objective $c$ subject to the
  linear inequalities given above. It is clear that an optimal
  solution to this LP is obtained by choosing a parent set
  $\varsubsetj$ for each $\vari \in \vertices$ such that $c_{\vari
    \leftarrow \varsubsetj}$ is maximal (or choosing none if all
  $c_{\vari \leftarrow \varsubsetj}$ are negative or there are no
  parent sets available). This is an integer
  solution. The result follows.
\end{proof}

\subsection{Sink-defined Faces}
\label{sec:sinkfaces}

For some particular $\varj \in \vertices$, consider the valid
inequality $\sum_{\vari\neq\varj,\varj \in \varsubsetj} x_{\vari \leftarrow
  \varsubsetj} \geq 0$. This
 defines a face $\fvpolyys{\vertices}{\varj}$ of the family
variable polytope as
\begin{equation}
  \label{eq:sinkface}
  \fvpolyys{\vertices}{\varj} := \Bigl\{ x \in \fvpolyy{\vertices,\ppsvalone} \Bigm| \sum_{\varj \in \varsubsetj, \vari\neq\varj} x_{\vari \leftarrow
  \varsubsetj} = 0\Bigr\}.
\end{equation}
This face contains all acyclic digraphs for which $\varj$ is a
\emph{sink}---it has no children. Since every acyclic digraph has at
least one sink, each extreme point of the family variable polytope
$\fvpolyy{\vertices,\ppsvalone}$ lies on a face $\fvpolyys{\vertices}{\varj}$ for
at least one choice of $\varj$.
\begin{remark}
  Abbreviate $|V|$ to $p$ and recall that
  $\mathrm{dim}(\fvpolyy{\vertices,\ppsvalone}) = p(2^{p-1}-1)$. We have that
  $\mathrm{dim}(\fvpolyys{\vertices}{\varj}) =
  \mathrm{dim}(\fvpolyy{\vertices \setminus \{\varj\},{\cal
      P}_{\vertices \setminus \{\varj\}}}) + 2^{p-1} - 1
  = (p-1)(2^{p-2}-1) + 2^{p-1} - 1 = (p+1)2^{p-2} - p$.
  If the family variables
  clamped to zero in $\fvpolyys{\vertices}{\varj}$ are removed,
  $\fvpolyys{\vertices}{\varj}$ is full-dimensional in $\reals^{(p+1)2^{p-2} - p}$.
\end{remark}

\begin{remark}
\label{rem:insinkface}
Every acyclic digraph contains at least one sink.
  So if $x$ is an extreme point of $\fvpolyy{\vertices,\ppsvalone}$, then $x \in
  \bigcup_{\varj \in \vertices}
      \fvpolyys{\vertices}{\varj}$.
\end{remark}

\begin{proposition}
  \label{prop:sink}
  The facet-defining inequalities of the full-dimensional polytope
  $\fvpolyys{\vertices}{\varj} \subseteq \reals^{(p+1)2^{p-2} - p}$ are
  \begin{enumerate}
  \item the facet-defining inequalities of the polytope $\fvpolyy{\vertices \setminus \{\varj\},{\cal
      P}_{\vertices \setminus \{\varj\}}}$, and 
  \item the modified convexity constraint for $\varj$, namely
    $\sum_{\varsubsetj \subseteq \vertices \setminus \{\varj\},
      \varsubsetj \neq \emptyset} x_{\varj \leftarrow \varsubsetj} \leq 1$.
  \end{enumerate}
\end{proposition}
\begin{proof}
  Let $c \in \reals^{(p+1)2^{p-2} - p}$ be an arbitrary objective coefficient
  vector and consider solving the LP with objective $c$ subject to the
  linear inequalities given above. Since $\varj$ is constrained to be
  a sink, an optimal solution in
  $\fvpolyys{\vertices}{\varj}$ is obtained by choosing a maximally
  scoring parent set for $\varj$ and then an optimal acyclic digraph
  for $\vertices \setminus \{\varj\}$. Since we have all the facets of
  the polytope $\fvpolyy{\vertices \setminus \{\varj\}}$, the optimal
  acyclic digraph for $\vertices \setminus \{\varj\}$ is a maximal
  solution to the LP restricted to the relevant variables. So the full
  LP has an integer solution. The result follows.
\end{proof}

\subsection{A Sink-based Extended Representation for the Family Variable Polytope}
\label{sec:extended}

Since $\fvpolyys{\vertices}{\varj} \subseteq \fvpolyy{\vertices,\ppsvalone}$, for
each $\varj \in \vertices$ we have $\bigcup_{\varj \in \vertices}
\fvpolyys{\vertices}{\varj} \subseteq \fvpolyy{\vertices,\ppsvalone}$ and so $\conv\left(
  \bigcup_{\varj \in \vertices}\fvpolyys{\vertices}{\varj}\right) \subseteq \conv
  (\fvpolyy{\vertices,\ppsvalone}) = \fvpolyy{\vertices,\ppsvalone}$. However, as noted 
in Remark~\ref{rem:insinkface}, if $x$ is an extreme point of
$\fvpolyy{\vertices,\ppsvalone}$, then $x \in \bigcup_{\varj \in \vertices}
\fvpolyys{\vertices}{\varj}$, so $\fvpolyy{\vertices,\ppsvalone}
\subseteq \conv\left(
  \bigcup_{\varj \in \vertices}\fvpolyys{\vertices}{\varj}\right)$, and thus $\fvpolyy{\vertices,\ppsvalone}
= \conv\left(
  \bigcup_{\varj \in \vertices}\fvpolyys{\vertices}{\varj}\right)$. Since there are only $|V|=p$
sink-defined faces, this leads to a compact extended representation for
the family variable polytope $\fvpolyy{\vertices,\ppsvalone}$ in terms of the
polytopes $\fvpolyys{\vertices}{\varj}_{\varj \in \vertices}$. Since
by Proposition~\ref{prop:sink} each $\fvpolyys{\vertices}{\varj}$
can be defined using $\fvpolyy{\vertices \setminus \{\varj\}}$, this
allows $\fvpolyy{\vertices,\ppsvalone}$ to be defined by the $\fvpolyy{\vertices
  \setminus \{\varj\},{\cal P}_{\vertices
  \setminus \{\varj\}}}$.
In Appendix~\ref{sec:liftandproject} we detail how this is done for
the specific case of $|\vertices|=4$; here we describe the method for
the general case.

A union of polytopes can be modelled by introducing additional
variables. We follow the (standard) approach described by \shortciteA[\S 2.11]{conforti14:integ-progr}. For each $\varj
\in \vertices$, we introduce a binary variable $x_{\varj}$ and add
the constraint 
\begin{equation}
  \label{eq:sinkconvex}
\sum_{\varj \in \vertices} x_{\varj} = 1,
\end{equation}
where $x_{\varj}$ indicates that node $\varj$ is a distinguished
sink. The
constraint states that in each acyclic digraph we can choose exactly
one sink as the distinguished sink for that digraph.

Next, for each $\varj \in \vertices$, $\vari \leftarrow \varsubsetj \in
\arcsets(\vertices,\ppsvalone)$, we introduce a new variable 
$x_{\varj,\vari \leftarrow \varsubsetj}$
indicating that $\vari$ has $\varsubsetj$ as its
(non-empty) parent set and that $\varj$ is the distinguished sink. In
other words $x_{\varj,\vari \leftarrow \varsubsetj} =
x_{\varj}x_{\vari \leftarrow \varsubsetj}$. We add the following
constraints linking the $x_{\varj,\vari \leftarrow \varsubsetj}$ to
the original $x_{\vari \leftarrow \varsubsetj}$:
\begin{equation}
  \label{eq:link}
  x_{\vari \leftarrow  \varsubsetj} = \sum_{\varj \in \vertices} x_{\varj,\vari \leftarrow
  \varsubsetj}.
\end{equation}
Denote the vector of $x_{\varj,\vari \leftarrow
  \varsubsetj}$ components for some $\varj$ as $x^{\varj}$. Then
for each $\varj \in \vertices$ and each facet-defining inequality $\pi x \leq
\pi_{0}$ of $\fvpolyys{\vertices}{\varj}$ we add the constraint
\begin{equation}
  \label{eq:altfacet}
  \pi^{\varj} x^{\varj} \leq \pi_{0}x_{\varj},
\end{equation}
where $\pi^{\varj}_{\varj,\vari \leftarrow
  \varsubsetj} = \pi_{\vari \leftarrow
  \varsubsetj}$, 
and also the variable bounds
\begin{equation}
  \label{eq:bounds}
  0 \leq x_{\varj,\vari \leftarrow \varsubsetj} \leq x_{\varj}.
\end{equation}

Equations and inequalities (\ref{eq:sinkconvex}--\ref{eq:bounds}) define $\bigcup_{\varj \in
  \vertices}\fvpolyys{\vertices}{\varj}$. To formulate
$\fvpolyy{\vertices,\ppsvalone} = \conv\bigl( \bigcup_{\varj \in
    \vertices}\fvpolyys{\vertices}{\varj}\bigr)$, it suffices to
merely drop the integrality condition on the $x_j$ variables, thus
allowing $\fvpolyy{\vertices,\ppsvalone}$ to be defined in terms of the
lower-dimensional $\fvpolyys{\vertices}{\varj}$.

\section{Relating BNSL and the Acyclic Subgraph Problem}
\label{sec:asp}

As the final contribution of this article, we establish a tight connection
between BNSL and the acyclic subgraph problem.

\subsection{BNSL as the Acyclic Subgraph Problem}

BNSL is closely related to the well-known \emph{acyclic subgraph problem} (\emph{ASP}) \shortcite{grotschel85}.
An instance of ASP is defined by digraph $\digraphname = \digraph$ with edge
weights $\arcweight{\vari}{\varj} \in \reals$ for every edge $\vari
\leftarrow \varj \in \arcs$, and the goal is to find an acyclic subdigraph $\digraphname' =
\digraphh$ of $\digraphname$ which maximises
\begin{equation}
  \label{eq:waspobjective}
  \sum_{\vari \leftarrow \varj \in \arcss} \arcweight{\vari}{\varj}.
\end{equation}

In ASP, the objective function is a linear function of
(indicators for) the edges of some digraph; in BNSL, by
contrast, the aim is to maximise an objective which is a linear
function of (indicators for) \emph{sets} of edges. As a Bayesian
network structure learning instance can consist of up to $\Omega(2^n)$
input values, it is presumably in general not possible to encode a BNSL instance as a ASP
instance over the same node set as the original BNSL instance, as this
would require in the worst case encoding an exponential number of
parent set scores into a quadratic number of edge weights. However, we
will next show that we can construct a BNSL-to-ASP reduction by
introducing new nodes to represent all possible parent sets of the
original instances $(\vertices, \ppsalone, \localscorevec)$, similarly
as in Theorem~\ref{thm:bnslp-bounded-parent-sets}. 

\begin{theorem}\label{thm:bnslp2asp}
Given BNSL instance $(\vertices_1, \ppsalone, \localscorevec)$, we can construct an ASP instance $D = (V,A)$ such that
\begin{enumerate}
    \item $\card{V} = O\bigl(\card{\vertices_1}+ \card{ \arcsetsfull(\vertices_1,\ppsalone) }\bigr)$, and
    \item there is one-to-one correspondence between the optimal solutions of $D$ and $(\vertices_1, \ppsalone, \localscorevec)$. 
\end{enumerate}
Moreover, given $(\vertices_1, \ppsalone, \localscorevec)$, the instance $D$ can be constructed in time $\operatorname{poly}\bigl(\card{\vertices_1}+ \card{ \arcsetsfull(\vertices_1,\ppsalone) }\bigr)$.
\end{theorem}

\begin{proof}

Define the digraph $\digraphname = \digraph$ where $V = V_{1} \cup
V_{2} \cup V_{3}$ and
\begin{itemize}
\item $V_{2} = \{\varsubsetj \subseteq V_{1} \mid \text{ $\varsubsetj \in \pps{\vari}$ for some $i \in \vertices_1$}\}$,
\item $V_{3} = \{ \vari \leftarrow
  \varsubsetj \mid \vari \in \vertices, \varsubsetj \in \pps{\vari} \}$.
\end{itemize}
See Figure~\ref{fig:graph} for an example node set where $V_1$ is on
the top row, $V_2$ the middle one and $V_3$ the bottom row.

The edge set for $\digraphname$ is the disjoint union of four
(colour-coded) edge sets $A = A_{1} \cup A_{2} \cup A_{3} \cup A_{4}$
where
\begin{itemize}
\item $A_{1} = \{(\vari,\varsubsetj) \mid \vari \in V_{1}, \varsubsetj \in V_{2}, \vari \in \varsubsetj\}$ (blue),
\item $A_{2} = \{(\vari \leftarrow \varsubsetj,\vari) \mid \vari
  \leftarrow \varsubsetj \in V_{3}, \vari \in V_{1} \}$ (black),
\item $A_{3} = \{(\varsubsetj,\vari \leftarrow \varsubsetj) \mid
  \varsubsetj \in V_{2}, \vari \leftarrow \varsubsetj \in V_{3}\}$ (red), and
\item $A_{4} = \{(\vari \leftarrow \varsubsetj,\varsubsetj') \mid \vari
  \leftarrow \varsubsetj \in V_{3}, \varsubsetj' \in V_{2}, \vari \not \in
    \varsubsetj', \varsubsetj \neq \varsubsetj'\}$ (green).
\end{itemize}
These four edge sets are coloured 
correspondingly in the example of Figure~\ref{fig:graph}.

Define an ASP instance for $\digraphname = \digraph$ where each (red)
edge in $A_3$ $(\varsubsetj,\vari \leftarrow \varsubsetj)$ has weight
$\localscore{\vari}{\varsubsetj}$; we will assume that the scores $c(i\leftarrow J)$ are strictly positive for all feasible parent sets choices, as adding the same value to each score will not change the optimal structures.
All other edges receive a weight
sufficiently big to ensure that they are included in any optimal acyclic
edge set. For example, giving each such edge a weight equal to a sum of
all $\localscore{\vari}{\varsubsetj}$ weights plus $1$ will suffice.

Note that $(\vertices,\arcs \setminus A_{3})$ is acyclic. Recall
also the objective coefficients of the ASP instance have been chosen
to ensure that $\arcs \setminus A_{3} \subseteq B$ for any optimal edge
set $B$ in $D$. Intuitively, we will thus only care about how the optimal solution looks on
the edge set $A_3$, and use this information to recover a solution to the original BNSL instance.

Let $(\vertices, B)$ be an optimal solution to ASP instance $D$ and define a digraph $(\vertices_1, B')$ as follows: $B' = \{ \vari \leftarrow \varj \mid \varj \in \varsubsetj \text{ and } (\varsubsetj,\vari \leftarrow \varsubsetj) \in B \}$.
We will show (i) there is exactly one edge of form $(\varsubsetj,\vari \leftarrow
  \varsubsetj)\in A_3 $ for each $i \in \vertices_1$, (ii) the graph $(\vertices_1, B')$
graph is acyclic and (iii) that it is an optimal solution to the given BNSL instance $(\vertices_1, \ppsalone, \localscorevec)$.

(i) Suppose that $(\varsubsetj,\vari \leftarrow \varsubsetj)$
and $(\varsubsetj',\vari \leftarrow \varsubsetj')$ were both in $B$
for some $\vari \in V_1$ and $\varsubsetj, \varsubsetj' \in V_{2},
\varsubsetj \neq \varsubsetj'$. This is not possible because the edges $(\vari
\leftarrow \varsubsetj,\varsubsetj')$ and $(\vari \leftarrow
\varsubsetj',\varsubsetj)$ are both in $A_{4}$ and thus in $B$. Having
$(\varsubsetj,\vari \leftarrow \varsubsetj)$
and $(\varsubsetj',\vari \leftarrow \varsubsetj')$ both in $B$
 would cause a cycle $(\vari \leftarrow \varsubsetj) \rightarrow
\varsubsetj \rightarrow (\vari \leftarrow \varsubsetj') \rightarrow
\varsubsetj \rightarrow (\vari \leftarrow \varsubsetj)$
in $B$, and so is impossible. 

(ii) For any $\vari, \varj, \varsubsetj$ with $\vari \neq \varj, \varj
\in \varsubsetj$, there exist the following edges: the blue edge
$(\varj,\varsubsetj) \in A_{1}$ and the black edge $(\vari \leftarrow
\varsubsetj,\vari) \in A_{2}$. Note that both of these edges will be in
$B$.  If the red edge $(\varsubsetj,\vari \leftarrow \varsubsetj) \in
A_{3}$ is also in $B$ then we have the following path in $B$: $\varj
\rightarrow \varsubsetj \rightarrow (\vari \leftarrow \varsubsetj)
\rightarrow \vari$. So if $\varj$ is a parent of $\vari$ in $B'$, then
there is a path from $\varj$ to $\vari$ in $B$. So if there were a
cycle $\vari_{1} \rightarrow \vari_{2} \dots \vari_{n} \rightarrow
\vari_{1}$ in $B'$ there would be a cycle from $\vari_{1}$ to
$\vari_{1}$ in $B$. Since $B$ is acyclic this is a contradiction and
so $B'$ must also be acyclic.

(iii) We first show that any feasible solution $G$ to the BNSL instance $(\vertices_1, \ppsalone, \localscorevec)$
corresponds to a feasible solution to the ASP instance $\digraphname = \digraph$. This feasible solution to $\digraphname = \digraph$ consists of the
edges $A_{1} \cup A_{2} \cup A_{4}$ together with those red edges in
$A_3$ corresponding to the parent set choices for $G$. We need to show
that this edge set---call it $B(G)$---is acyclic in $\digraphname =
\digraph$. Since $G$ is acyclic there is a total order $<_{V_{1}}$ on
the nodes $V_1$ such that parents always come before children in
this order. We show that $<_{V_{1}}$ determines a total order $<_{V}$
on the nodes $V$ such that parents always come before children in
$B(G)$ which establishes that $B(G)$ is acyclic.

To aid understanding we first do this for the case where $V_{1} =
\{a,b,c\}$ and $G$ is such that $a <_{V_{1}} b <_{V_{1}} c$. The
general result is established later. In the special case all red edges
(in $A_{3}$) which are inconsistent with $<_{V_{1}}$ will be absent
from $B(G)$. In particular, since $a$ is allowed no parents, the red
edges going to the nodes $a\leftarrow\{b\}$, $a\leftarrow\{c\}$,
and $a\leftarrow\{b,c\}$ will be absent. This means that these
nodes are source nodes in $B(G)$, so put these as the first 3
elements of the order $<_{V}$. Since $a \in V_{1}$ has only these 3
nodes as parents, put $a$ as the 4th element in $<_{V}$. Since the
only parent for $\{a\}$ is $a$, put $\{a\}$ as the 5th element. 
Since $c$ cannot be a parent of $b$, the red arrows going to 
$b \leftarrow \{c\}$ and $b \leftarrow \{a,c\}$ are absent from
$B(G)$, so these nodes are sources in $B(G)$. Also the only arrow
going to $b \leftarrow \{a\}$ is from $\{a\}$ which is already in the
order. This allows us to put $b \leftarrow \{a\}$, $b \leftarrow
\{c\}$ and $b \leftarrow \{a,c\}$ as the next elements in
$<_{V}$. Having done this $b$ can be placed next, and then $\{b\}$
and $\{a,b\}$. The final placements are $c \leftarrow \{a\}$, $c \leftarrow
\{b\}$ and $c \leftarrow \{a,b\}$, then $c$ and then the remaining
nodes $\{c\}, \{a,c\}, \{b,c\}$ and $\{a,b,c\}$. 

In the general case, suppose we have $G$ with a consistent ordering of
its nodes $\vari_{1} <_{V_{1}} \vari_{2} \dots <_{V_{n}}
\vari_{n}$. We construct a total ordering of the nodes of $V$
consistent with $B(G)$ as follows. Start with the nodes $\vari_{1}
\leftarrow \varsubsetj$ (in any order), and then put $\vari$ and after
that $\{\vari\}$. Then for $k=2, \dots, n$ add nodes as follows:
the $\vari_{k} \leftarrow \varsubsetj$ nodes, then $\vari_{k}$ and
then each $\varsubsetj$ such that $\vari_{k} \in \varsubsetj$ and
$\varsubsetj \subseteq \{\vari_{1} \dots \vari_{k}\}$. It is not
difficult to see that this total order contains all nodes of $V$
and is consistent with $B(G)$, so $B(G)$ is acyclic.

Now suppose $B'$ were not an optimal solution to the BNSL instance $(\vertices_1, \ppsalone, \localscorevec)$. In that
case there would be some strictly better solution corresponding to an
acyclic graph $G$ for which $B(G)$ would be a feasible solution to the
ASP instance $\digraphname = \digraph$ and this solution would be strictly better than $B$. This
is a contradiction since $B$ is an optimal solution and so it
follows that $B'$ is an optimal solution to $(\vertices_1, \ppsalone, \localscorevec)$.
\end{proof}

\begin{figure}
  \centering
  \includegraphics[width=1.05\textwidth]{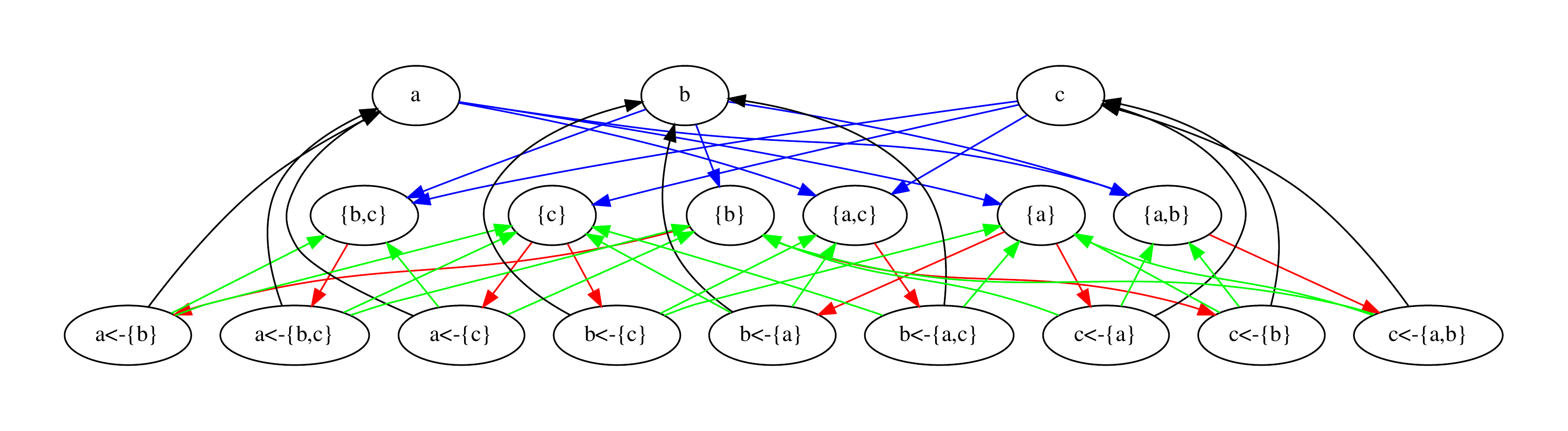}
\vspace{-1cm}
  \caption{ASP digraph for the BNSL instance with node set $\{a,b,c\}$}
  \label{fig:graph}
\end{figure}

Note that, as \citeA{marti11:_linear_order_probl} show,  ASP is equivalent to the
\emph{linear ordering problem} (LOP).
This
means that pure LOP approaches can be used to solve ASP and thus BNSL.

\subsection{Relating the BNSL and Acyclic Subgraph Problem Polytopes}
There is a
polytope naturally associated with any instance of ASP.
Let $\reals^{\arcs}$ be a real vector space where every component of a
vector $y \in \reals^{\arcs}$ is indexed by an edge $\vari \leftarrow \varj \in
\arcs$. For every edge set $\arcss \subseteq \arcs$, the incidence
vector  $y^{\arcss} \in
  \reals^{\arcs}$ of $\arcss$ is defined by
$y^{\arcss}_{\vari \leftarrow \varj} =  1$ if $\vari \leftarrow \varj \in \arcss$ and
$y^{\arcss}_{\vari \leftarrow \varj} =  0$ if $\vari \leftarrow \varj \not\in \arcss$.
The \emph{acyclic subgraph polytope} $\asppoly{\digraphname}$ is
\begin{equation}
  \label{eq:pac}
  \asppoly{\digraphname} := \mathrm{conv}\bigl\{y^{\arcss} \in
  \reals^{\arcs} \bigm| \arcss \in {\cal \arcs}(\digraphname)\bigr\}.
\end{equation}

It is not difficult to see that the all-parent-sets-allowed family variable polytope
$\fvpolyy{\vertices,\ppsvalone}$ can be projected onto the ASP polytope where the
ASP edgeset $\arcs = \vertices \times \vertices$. Equivalently,
BNSL is an \emph{extended formulation} of such ASP instances. Since
the ASP has been extensively studied it is important to investigate
which results on  ASP `translate' to BNSL.

We can see that the ASP instance is a projection of
the BNSL instance by introducing the edge indicator variables $y_{\vari \leftarrow
  \varj}$ into BNSL together with the `linking' equations
\begin{equation}
  \label{eq:aspbnslp}
  y_{\vari \leftarrow \varj} = \sum_{\varsubsetj: \varj \in
    \varsubsetj} x_{\vari \leftarrow \varsubsetj}.
\end{equation}
The introduction of these variables (dimensions) and equations leaves
the family variable polytope unaltered except that it now `lives in' a
higher-dimensional space. `Projecting away' the $x_{\vari \leftarrow \varsubsetj}$
variables from this higher-dimensional family variable polytope then produces the ASP polytope.

Using this relationship it is easy to map any ASP instance with edgeset $\arcs =
\vertices \times \vertices$ into a
BNSL instance: simply  set $\localscore{\vari}{\varsubsetj} =
\sum_{\varj \in \varsubsetj} \arcweight{\vari}{\varj}$. A solution to
the BNSL instance so produced will be a solution to the original ASP instance
with the same objective value. A direct reverse mapping is only
possible if there are edge weights such that the local score for each
family is the sum of the weights of the edges corresponding to that
family.

\begin{proposition}
  If $\pi y \leq \pi_{0}$
is a valid inequality for ASP, then
$\pi' x \leq \pi_{0}$
is a valid inequality, where
$\pi'_{\vari \leftarrow \varsubsetj} = \sum_{\varj \in \varsubsetj}
\pi_{\vari \leftarrow \varj}.$
\end{proposition}
\begin{proof}
  Let $x^{*} \in \reals^{\arcsets}$ represent an acyclic digraph and
  let $y^{*} \in \reals^{\arcs}$ represent the same digraph.  We
  have that $\pi y^{*} \leq \pi_{0}$.  It is obvious that $\pi y^{*} =
  \pi' x^{*}$. So all acyclic digraphs represented by family variables
  satisfy $\pi' x \leq \pi_{0}$. The result follows.
\end{proof}

\section{Conclusions}
\label{sec:conclusions}

Integer programming, and specifically the IP-based \gobnilp{} system,
offers a state-of-the-art practical approach 
to the NP-hard optimization problem of 
learning optimal Bayesian network structures, BNSL.
Thus providing fundamental insights into the IP approach to BNSL
is important both from the purely scientific perspective---dealing
with a central class of probabilistic graphical models
with various applications in AI---and for 
developing a better understanding of the approach in the hope of further improving the 
current algorithmic approaches to BNSL.
With these motivations, in this work we shed light on various fundamental 
computational and representational aspects of BNSL. From the practical perspective,
many of our main contributions have tight connections to 
IP cutting planes derived in practice during  search for optimal network structures. 
Specifically, our contributions include for example the following.  
We showed that the separation problem
which in practice yields problem-specific BNSL cutting planes within \gobnilp{}
is in fact NP-hard, a previously open problem.
 We studied the relationship between three key polytopes
underlying BNSL. We analyzed the facets of the three polytopes,
and established  that the so-called cluster constraints giving rise to BNSL cutting planes
are in fact facet-defining inequalities of the family-variable
polytope central to BNSL. We also provide (in Appendix~\ref{sec:lowdim}) a complete enumeration of facets for low-dimensional family-variable polytopes,
connecting with problem-specific cutting planes ruling out all network structures
with short cyclic substructures. 
In summary, the theoretical results presented in this work deepen the current understanding 
of fundamental aspects of BNSL from various perspectives.

\acks{We thank three anonymous reviewers and the editor for useful
  criticism which has helped us improve the paper. The authors
  gratefully acknowledge financial support from: UK Medical Research
  Council Grant G1002312 (JC, MB); Senior Postdoctoral Fellowship
  SF/14/008 from KU Leuven (JC); UK NC3RS Grant NC/K001264/1 (JC);
  Academy of Finland under grants 251170 COIN Centre of Excellence in
  Computational Inference Research, 276412, and 284591 (MJ); Research
  Funds of the University of Helsinki (MJ); and Icelandic Research
  Fund grant 152679-051 (JHK). Part of the work was done while JHK was
  at University of Helsinki and at Reykjavik University.}

\vskip 0.2in
\bibliography{asasp}
\bibliographystyle{theapa}

\appendix

\section{Enumeration of Facets for Low-dimensional Family Variable Polytopes}
\label{sec:lowdim}

In Section~\ref{sec:facets} we provided general results on the facets
of the family variable polytope. In this section, 
we provide a complete listing of all facet-defining inequalities
(i.e.\ a minimal description of the convex hull by inequalities) of
the family variable polytope $\fvpolyy{\vertices,\ppsvalone}$ for
$|\vertices|=2,3,4$. 
We will observe that all
lower bounds on variables, modified convexity constraints and
$\palim$-cluster inequalities are indeed among the facets
found, as predicted by our theoretical results. Proposition~\ref{prop:nondecrease} and the lifting theorem
(Theorem~\ref{thm:extend}) are also consistent with the list of
facets. In the case of $|\vertices|=4$, we also see that there are
many facets \emph{not} given in Section~\ref{sec:facets}. In
Section~\ref{sec:v4k2} we enumerate all facet-defining inequalities
for $|\vertices|=4$, where at most two parents are allowed and observe
that the results are consistent with Theorem~\ref{thm:facetrestrict}.

We use $a,b,c$, and $d$ to label the nodes. To simplify notation,
we abbreviate variables such as
 $x_{b \leftarrow \{a,c\}}$ to
$x_{b \leftarrow ac}$.

\subsection{Node Set of Size 2}
\label{sec:v2}

When $|\vertices|=2$, there are 3 acyclic digraphs and $\arcsets(\vertices,\ppsvalone) = \{
a \leftarrow \{b\}, b \leftarrow \{a\} \}$. There are three facets: the two lower bounds and the
1-cluster constraint $x_{a \leftarrow b} + x_{b \leftarrow a} \leq 1$. 

\subsection{Node Set of Size 3}
\label{sec:v3}

When $|\vertices|=3$, there are 25 acyclic digraphs and 
\begin{align*}
\arcsets(\vertices,\ppsvalone) = \{ 
&a \leftarrow \{b\}, a \leftarrow\{c\}, a \leftarrow\{b,c\}, \\
&b \leftarrow\{a\}, b \leftarrow\{c\}, b \leftarrow\{a,c\}, \\
&c \leftarrow\{a\}, c \leftarrow\{b\}, c \leftarrow\{a,b\} \}.
\end{align*}
Using the \texttt{cdd} computer program \shortcite{fukuda15:_homep}, we found all the facets of the
convex hull of the 25 acyclic digraphs. There are 17 facet-defining inequalities:
\begin{itemize}
\item 9 lower bounds on the 9 $x_{\vari \leftarrow \varsubsetj}$
  family variables;
\item 3 modified convexity constraints, one for each of $a$, $b$, and $c$;
\item 4 1-cluster constraints, one for each of the  clusters
  $\{a,b\}$, $\{a,c\}$, $\{b,c\}$, and $\{a,b,c\}$; and
\item 1 2-cluster constraint for the cluster $\{a,b,c\}$.
\end{itemize}

\subsection{Node Set of Size 4}
\label{sec:v4}

When $|\vertices|=4$, there are 543 acyclic digraphs and 
\begin{align*}
\arcsets(\vertices,\ppsvalone) = \{ 
&a\leftarrow\{b\}, a\leftarrow\{c\}, a\leftarrow\{d\}, a\leftarrow\{b,c\}, a\leftarrow\{b,d\},
a\leftarrow\{c,d\}, a\leftarrow\{b,c,d\}, \\
&b \leftarrow \{a\}, b \leftarrow \{c\}, b \leftarrow \{d\}, b \leftarrow \{a,c\}, b \leftarrow \{a,d\},
b \leftarrow \{c,d\}, b \leftarrow \{a,c,d\}, \\
&c \leftarrow \{a\}, c \leftarrow \{b\}, c \leftarrow \{d\}, c \leftarrow \{a,b\}, c \leftarrow \{a,d\},
c \leftarrow \{b,d\}, c \leftarrow \{a,b,d\}, \\
&d \leftarrow \{a\}, d \leftarrow \{b\}, d \leftarrow \{c\}, d
\leftarrow \{a,b\}, d \leftarrow \{a,c\}, d \leftarrow \{b,c\}, d
\leftarrow \{a,b,c\} \}.
\end{align*}
Using \texttt{cdd} we discovered that there are 135 facet-defining inequalities of the
family variable polytope:
\begin{itemize}
\item 28 lower bounds on the 28 $x_{\vari \leftarrow \varsubsetj}$
  family variables;
\item 4 modified convexity constraints, one for each of $a$, $b$, $c$, and $d$;
\item 6 1-cluster constraints for each of the $\binom{4}{2} = 6$
  clusters of size 2; 
\item 4 1-cluster constraints for each of the $\binom{4}{3} = 4$
  clusters of size 3; 
\item 1 1-cluster constraint for the $\binom{4}{4} = 1$
  cluster of size 4; 
\item 4 2-cluster constraints for each of the $\binom{4}{3} = 4$
  clusters of size 3; 
\item 1 2-cluster constraint for the $\binom{4}{4} = 1$
  cluster of size 4; 
\item 1 3-cluster constraint for the $\binom{4}{4} = 1$
  cluster of size 4; and 
\item 86 other facet-defining inequalities.
\end{itemize}

We now list these 86 other facet-defining inequalities. These 86 inequalities fall into 9
permutation classes, and we give just one member of each of
these 9 classes. By symmetry, any permutation of the 4 nodes
$a$, $b$, $c$, and $d$ in a facet-defining inequality will produce
another facet-defining inequality. Some permutations do not change the inequality. We
indicate this, for each permutation class, by showing which nodes
can be permuted without changing the facet. For example, the expression
$ab|cd$ indicates that either $a$ and $b$, or $c$ and $d$, can
be swapped without altering the inequality, so that there are $4!/(2\times
2) = 6$ distinct inequality in such a permutation class.

For each permutation class, we give the (arbitrarily chosen) name for
that class that is used by the \gobnilp{} system. The names run from 4B
to 4I---there is no permutation class called `4A', since, at one time
in \gobnilp, this was used to designate $\palim$-cluster inequalities.
With the exception of `4F' and `4J' inequalities, if the user wants,
\gobnilp{} can search for these facets as cutting planes for a given LP
solution. By default only `4B' cutting planes are looked for, since
these cutting planes have empirically been found to perform
well. Interestingly, 4B facets can be defined in terms of connected
matroids, as noted by \citeA{studeny15:_how_bayes}.

\begin{description}
\item[4B facets $ab|cd$]
  \begin{align}\label{eq:4b}
   &x_{a \leftarrow b} + x_{a\leftarrow bc} + x_{a \leftarrow bd} +
   x_{a \leftarrow cd} + x_{a \leftarrow bcd}& \nonumber \\ 
 + &x_{b \leftarrow a} + x_{b \leftarrow ac} + x_{b \leftarrow ad} +
 x_{b \leftarrow cd} + x_{b \leftarrow acd} \nonumber \\
 + &x_{c \leftarrow ad} + x_{c \leftarrow bd} + x_{c \leftarrow abd}
 \nonumber \\
 + &x_{d \leftarrow ac} + x_{d \leftarrow bc} + x_{d \leftarrow abc} &\leq 2
\end{align}
6 inequalities

\item[4C facets $a|b|cd$]
  \begin{align*}
&x_{a \leftarrow c} + x_{a \leftarrow d} + x_{a \leftarrow bc} + x_{a \leftarrow bd} + x_{a \leftarrow cd} + x_{a \leftarrow bcd}&\\ 
+ &x_{b \leftarrow cd} + x_{b \leftarrow acd}\\
 + &x_{c \leftarrow ab} + x_{c \leftarrow bd} + x_{c \leftarrow abd}\\
 + &x_{d \leftarrow ab} + x_{d \leftarrow bc} + x_{d \leftarrow abc} &\leq 2\\
  \end{align*}
12 inequalities

\item[4D facets $a|b|cd$]
  \begin{align*}
 &x_{a \leftarrow b} + x_{a \leftarrow c} + x_{a \leftarrow d} + x_{a \leftarrow bc} + x_{a \leftarrow bd} + 2 x_{a \leftarrow cd} + 2
 x_{a \leftarrow bcd}&\\
 + &x_{b \leftarrow a} + x_{b \leftarrow c} + x_{b \leftarrow d} + x_{b \leftarrow ac} + x_{b \leftarrow ad} + x_{b \leftarrow cd} +
 x_{b \leftarrow acd}\\ 
+ &x_{c \leftarrow a} + x_{c \leftarrow ab} + x_{c \leftarrow ad} + x_{c \leftarrow abd}\\ 
+ &x_{d \leftarrow a} + x_{d \leftarrow ab} + x_{d \leftarrow ac} + x_{d \leftarrow abc} &\leq 3\\
  \end{align*}
12 inequalities

\item[4E facets $a|bcd$]
  \begin{align*}
  &x_{a \leftarrow bc} + x_{a \leftarrow bd} + x_{a \leftarrow cd} + 2x_{a \leftarrow bcd}&\\ 
+ &x_{b \leftarrow ac} + x_{b \leftarrow ad} + x_{b \leftarrow acd}\\ 
+ &x_{c \leftarrow ab} + x_{c \leftarrow ad} + x_{c \leftarrow abd}\\
+ &x_{d \leftarrow ab} + x_{d \leftarrow ac} + x_{d \leftarrow abc} &\leq 2
\end{align*}
4 inequalities

\item[4F facets $ab|cd$]
  \begin{align*}
&x_{a \leftarrow cd} + x_{a \leftarrow bcd}\\ 
+ &x_{b \leftarrow cd} + x_{b \leftarrow acd}\\
+ &x_{c \leftarrow a} + x_{c \leftarrow b} + x_{c \leftarrow d} + x_{c \leftarrow ab} + x_{c \leftarrow ad} + x_{c \leftarrow bd} + 2
x_{c \leftarrow abd}\\
+ &x_{d \leftarrow a} + x_{d \leftarrow b} + x_{d \leftarrow c} + x_{d \leftarrow ab} + x_{d \leftarrow ac} + x_{d \leftarrow bc} + 2 x_{d \leftarrow abc} \leq 3\\
  \end{align*}
6 inequalities

\item[4G facets $a|b|c|d$]
  \begin{align*}
&x_{a \leftarrow cd} + x_{a \leftarrow bcd} \\
+ &x_{b \leftarrow c} + x_{b \leftarrow ac} + x_{b \leftarrow cd} + x_{b \leftarrow acd} \\ 
+ &x_{c \leftarrow b} + x_{c \leftarrow d} + x_{c \leftarrow ab} + x_{c \leftarrow ad} + x_{c \leftarrow bd} + 2 x_{c \leftarrow abd}\\
+ &x_{d \leftarrow a} + x_{d \leftarrow b} + x_{d \leftarrow c} + x_{d \leftarrow ab} + 2 x_{d \leftarrow ac} + x_{d \leftarrow bc} + 2
x_{d \leftarrow abc} \leq 3
  \end{align*}
24 inequalities

\item[4H facets $a|b|cd$]
  \begin{align*}
&x_{a \leftarrow c} + x_{a \leftarrow d} + x_{a \leftarrow bc} + x_{a \leftarrow bd} + x_{a \leftarrow cd} + 2 x_{a \leftarrow bcd}&\\
 + &x_{b \leftarrow acd}\\
 + &x_{c \leftarrow ab} + x_{c \leftarrow abd}\\ 
+ &x_{d \leftarrow ab} + x_{d \leftarrow abc} &\leq 2\\
  \end{align*}
12 inequalities

\item[4I facets $ab|cd$]
  \begin{align*}
&x_{a \leftarrow c} + x_{a \leftarrow d} + x_{a \leftarrow bc} + x_{a \leftarrow bd} + x_{a \leftarrow cd} + 2 x_{a \leftarrow bcd}&\\
 + &x_{b \leftarrow c} + x_{b \leftarrow d} + x_{b \leftarrow ac} + x_{b \leftarrow ad} + x_{b \leftarrow cd} + 2 x_{b \leftarrow acd}\\
 + &x_{c \leftarrow a} + x_{c \leftarrow b} + x_{c \leftarrow d} + 2 x_{c \leftarrow ab} + x_{c \leftarrow ad} + x_{c \leftarrow bd} +
 2x_{c \leftarrow abd}\\ 
+ &x_{d \leftarrow a} + x_{d \leftarrow b} + x_{d \leftarrow c} + 2 x_{d \leftarrow ab} + x_{d \leftarrow ac} + x_{d \leftarrow bc} + 2 x_{d \leftarrow abc}& \leq 4
  \end{align*}
6 inequalities

\item[4J facets $a|bcd$]
  \begin{align*}
&x_{a \leftarrow b} + x_{a \leftarrow c} + x_{a \leftarrow d} + 2 x_{a \leftarrow bc} + 2 x_{a \leftarrow bd} + 2 x_{a \leftarrow cd} +
2x_{a \leftarrow bcd}&\\
 + &x_{b \leftarrow a} + x_{b \leftarrow ac} + x_{b \leftarrow ad} + x_{b \leftarrow cd} + x_{b \leftarrow acd}\\ 
+ &x_{c \leftarrow a} + x_{c \leftarrow ab} + x_{c \leftarrow ad} + x_{c \leftarrow bd} + x_{c \leftarrow abd}\\ 
+ &x_{d \leftarrow a} + x_{d \leftarrow ab} + x_{d \leftarrow ac} + x_{d \leftarrow bc} + x_{d \leftarrow abc} &\leq 3
  \end{align*}
4 inequalities

\end{description}

\subsection{Node Set of Size 4, Parent Set Size at most 2}
\label{sec:v4k2}

By Theorem~\ref{thm:facetrestrict}, if we have 4 nodes but
only allow acyclic digraphs with at most two parents, then the
following facet-defining inequalities 
from Section~\ref{sec:v4} (with family variables 
$x_{a\leftarrow  bcd}$, 
$x_{b\leftarrow  acd}$, 
$x_{c\leftarrow  abd}$, and $x_{d\leftarrow  abc}$ removed)  should
be facet-defining inequalities of the resulting polytope.

\begin{itemize}
\item 24 lower bounds on the 24 $x_{\vari \leftarrow \varsubsetj}$
  family variables;
\item 4 modified convexity constraints, one for each of $a$, $b$, $c$, and $d$;
\item 6 1-cluster constraints for each of the $\binom{4}{2} = 6$
  clusters of size 2; 
\item 4 1-cluster constraints for each of the $\binom{4}{3} = 4$
  clusters of size 3; 
\item 1 1-cluster constraint for the $\binom{4}{4} = 1$
  cluster of size 4; 
\item 4 2-cluster constraints for each of the $\binom{4}{3} = 4$
  clusters of size 3; and 
\item 1 2-cluster constraint for the $\binom{4}{4} = 1$
  cluster of size 4. 
\end{itemize}

In addition all facet-defining inequalities of types 4B, 4C, 4D, and
4J should remain facet-defining. There are 6, 12, 12, and 4 of these,
respectively. This  adds up to a total of
24+4+6+4+1+4+1+6+12+12+4=78 facet-defining inequalities. Using
\texttt{cdd} we computed the facet-defining inequalities of the convex
hull of the (family-variable encoded) 443 acyclic digraphs with 4 nodes
and where each node has at most 2 parents. We found, as expected, that
all of these 78 inequalities were included. Moreover, we found that
these 78 constitute the \emph{complete set} of facet-defining
inequalities---there are no others.

\section{Lift-and-Project for Family Variable Polytopes}
\label{sec:liftandproject}

In this appendix, we apply a `lift-and-project' method based on the
sink-based extended representation of Section~\ref{sec:extended} to
derive a representation of $\fvpolyy{\{a,b,c,d\},{\cal
    P}_{\{a,b,c,d\}}}$, whose facet-defining inequalities are given in
Section~\ref{sec:v4}, in terms of $\fvpolyy{\{b,c,d\},{\cal
    P}_{\{b,c,d\}}}$, $\fvpolyy{\{a,c,d\},{\cal P}_{\{a,c,d\}}}$,
$\fvpolyy{\{a,b,d\},{\cal P}_{\{a,b,d\}}}$ and
$\fvpolyy{\{a,b,c\},{\cal P}_{\{a,b,c\}}}$, whose facet-defining
inequalities are given in Section~\ref{sec:v3}. First we have the relevant formulation of \eqref{eq:sinkconvex},
\begin{equation}
  \label{eq:sinkconvex4}
  x_{a} + x_{b} + x_{c} + x_{d} = 1,
\end{equation}
stating that exactly one of the four nodes is the distinguished sink in
any acyclic digraph using those four nodes.
Recall that $x_{\varj,\vari \leftarrow \varsubsetj}$ indicates that
$\varj$ is the distinguished sink and that $\varsubsetj$ is the parent
set for $\vari$ so that $x_{\varj,\vari \leftarrow \varsubsetj} = 0$ if $\varj \in
\varsubsetj$, so that, for example, $x_{b,a \leftarrow b}=0$. With
this observation we can write the linking equations (\ref{eq:link})
as follows.
\begin{align}
\label{eq:eq1}  x_{a \leftarrow b} &=  x_{a,a \leftarrow b} +  x_{c,a \leftarrow b}  +  x_{d,a \leftarrow b}\\
  x_{a \leftarrow c} &=  x_{a,a \leftarrow c} +  x_{b,a \leftarrow c}  +  x_{d,a \leftarrow c}\\
  x_{a \leftarrow d} &=  x_{a,a \leftarrow d} +  x_{b,a \leftarrow d}  +  x_{c,a \leftarrow d}\\
  x_{b \leftarrow a} &=  x_{b,b \leftarrow a} +  x_{c,b \leftarrow a}  +  x_{d,b \leftarrow a}\\
  x_{b \leftarrow c} &=  x_{a,b \leftarrow c} +  x_{b,b \leftarrow c}  +  x_{d,b \leftarrow c}\\
  x_{b \leftarrow d} &=  x_{a,b \leftarrow d} +  x_{b,b \leftarrow d}  +  x_{c,b \leftarrow d}\\
  x_{c \leftarrow a} &=  x_{b,c \leftarrow a} +  x_{c,c \leftarrow a}  +  x_{d,c \leftarrow a}\\
  x_{c \leftarrow b} &=  x_{a,c \leftarrow b} +  x_{c,c \leftarrow b}  +  x_{d,c \leftarrow b}\\
  x_{c \leftarrow d} &=  x_{a,c \leftarrow d} +  x_{b,c \leftarrow d}  +  x_{c,c \leftarrow d}\\
  x_{d \leftarrow a} &=  x_{b,d \leftarrow a} +  x_{c,d \leftarrow a}  +  x_{d,d \leftarrow a}\\
  x_{d \leftarrow b} &=  x_{a,d \leftarrow b} +  x_{c,d \leftarrow b}  +  x_{d,d \leftarrow b}\\
  x_{d \leftarrow c} &=  x_{a,d \leftarrow c} +  x_{b,d \leftarrow c}  +  x_{d,d \leftarrow d}
\end{align}
\begin{align}
  x_{a \leftarrow bc} &=  x_{a,a \leftarrow bc} +  x_{d,a \leftarrow bc}\\
  x_{a \leftarrow bd} &=  x_{a,a \leftarrow bd} +  x_{c,a \leftarrow bd}\\
  x_{a \leftarrow cd} &=  x_{a,a \leftarrow cd} +  x_{b,a \leftarrow cd}\\
  x_{b \leftarrow ac} &=  x_{b,b \leftarrow ac} +  x_{d,b \leftarrow ac}\\
  x_{b \leftarrow ad} &=  x_{b,b \leftarrow ad} +  x_{c,b \leftarrow ad}\\
  x_{b \leftarrow cd} &=  x_{b,b \leftarrow cd} +  x_{a,b \leftarrow cd}\\
  x_{c \leftarrow ab} &=  x_{c,c \leftarrow ab} +  x_{d,c \leftarrow ab}\\
  x_{c \leftarrow ad} &=  x_{b,c \leftarrow ad} +  x_{c,c \leftarrow ad}\\
  x_{c \leftarrow bd} &=  x_{a,c \leftarrow bd} +  x_{c,c \leftarrow bd}\\
  x_{d \leftarrow ab} &=  x_{c,d \leftarrow ab} +  x_{d,d \leftarrow ab}\\
  x_{d \leftarrow ac} &=  x_{b,d \leftarrow ac} +  x_{d,d \leftarrow ac}\\
\label{eq:eq24}  x_{d \leftarrow bc} &=  x_{a,d \leftarrow bc} +  x_{d,d \leftarrow bc}
\end{align}
\begin{align}
\label{eq:reda}    x_{a \leftarrow bcd} &=  x_{a,a \leftarrow bcd}\\
\label{eq:redb}    x_{b \leftarrow acd} &=  x_{b,b \leftarrow acd}\\
\label{eq:redc}    x_{c \leftarrow abd} &=  x_{c,c \leftarrow abd}\\
\label{eq:redd}    x_{d \leftarrow abc} &=  x_{d,d \leftarrow abc}
\end{align}
Thirdly we have all equations of type (\ref{eq:altfacet}). We label
all inequalities with $x_{a}$ on the RHS as follows. The modified
convexity constraints for $a$, $b$, $c$, and $d$ are labelled a-a, a-b,
a-c, and a-d, respectively. All other constraints are cluster
constraints which we label as a-C, where $C$ is the cluster and
$\palim=1$, and a-2-C, where $C$ is the cluster and
$\palim=2$. Inequalities with $x_b$, $x_c$, and $x_d$ on the RHS are labelled
analogously. The 36 inequalities of type (\ref{eq:altfacet}) are now
listed using this labelling convention.
\begin{align}
\label{eq:a-b}  \tag{a-b} x_{a,b \leftarrow c} + x_{a,b \leftarrow d} + x_{a,b \leftarrow cd} &\leq x_{a}\\
\label{eq:a-c} \tag{a-c}  x_{a,c \leftarrow b} + x_{a,c \leftarrow d} + x_{a,c \leftarrow bd} &\leq x_{a}\\
\label{eq:a-d} \tag{a-d}  x_{a,d \leftarrow b} + x_{a,d \leftarrow c} + x_{a,d \leftarrow cd} &\leq x_{a}\\
\label{eq:a-bc} \tag{a-bc}  x_{a,b \leftarrow c} + x_{a,b \leftarrow cd} + x_{a,c \leftarrow b} + x_{a,c \leftarrow bd} &\leq x_{a}\\
\label{eq:a-bd} \tag{a-bd}  x_{a,b \leftarrow d} + x_{a,b \leftarrow cd} + x_{a,d \leftarrow b} + x_{a,d \leftarrow bd} &\leq x_{a}\\
\label{eq:a-cd} \tag{a-cd}  x_{a,c \leftarrow d} + x_{a,c \leftarrow bd} + x_{a,d \leftarrow c} + x_{a,d \leftarrow cd} &\leq x_{a}\\
  x_{a,b \leftarrow c} + x_{a,b \leftarrow d} + x_{a,b \leftarrow cd}
  +
  x_{a,c \leftarrow b} + x_{a,c \leftarrow d} + x_{a,c \leftarrow bd} & \nonumber \\
\label{eq:a-bcd} \tag{a-bcd}  + x_{a,d \leftarrow b} + x_{a,d \leftarrow c} + x_{a,d \leftarrow cb}  &\leq 2x_{a}\\
\label{eq:a-2-bcd} \tag{a-2-bcd}  x_{a,b \leftarrow cd} + x_{a,c \leftarrow bd} + x_{a,d \leftarrow
    bc}  &\leq x_{a}\\
  x_{a,a \leftarrow b} + x_{a,a \leftarrow c} + x_{a,a \leftarrow d} +
  x_{a,a \leftarrow bc} & \nonumber \\ \label{eq:a-a} \tag{a-a} + x_{a,a \leftarrow bd} +
  x_{a,a \leftarrow cd} + x_{a,a \leftarrow bcd} &\leq x_{a}
\end{align}
\begin{align}
\label{eq:b-a} \tag{b-a}  x_{b,a \leftarrow c} + x_{b,a \leftarrow c} + x_{b,a \leftarrow cd} &\leq x_{b}\\
\label{eq:b-c} \tag{b-c}  x_{b,c \leftarrow a} + x_{b,c \leftarrow d} + x_{b,c \leftarrow ad} &\leq x_{b}\\
\label{eq:b-d} \tag{b-d}  x_{b,d \leftarrow a} + x_{b,d \leftarrow c} + x_{b,d \leftarrow ac} &\leq x_{b}\\
\label{eq:b-ac} \tag{b-ac}  x_{b,a \leftarrow c} + x_{b,a \leftarrow cd} + x_{b,c \leftarrow a} + x_{b,c \leftarrow ad} &\leq x_{b}\\
\label{eq:b-ad} \tag{b-ad}  x_{b,a \leftarrow d} + x_{b,a \leftarrow cd} + x_{b,d \leftarrow a} + x_{b,d \leftarrow ac} &\leq x_{b}\\
\label{eq:b-cd} \tag{b-cd}  x_{b,c \leftarrow d} + x_{b,c \leftarrow bd} + x_{b,d \leftarrow c} + x_{b,d \leftarrow cd} &\leq x_{b}\\
  x_{b,a \leftarrow c} + x_{b,a \leftarrow d} + x_{b,a \leftarrow cd} + 
  x_{b,c \leftarrow a} + x_{b,c \leftarrow d} + x_{b,c \leftarrow ad} & \nonumber \\ 
\label{eq:b-acd} \tag{b-acd}  + x_{b,d \leftarrow a} + x_{b,d \leftarrow c} + x_{b,d \leftarrow ac}  &\leq 2x_{b}\\
\label{eq:b-2-acd} \tag{b-2-acd}  x_{b,a \leftarrow cd} + x_{b,c \leftarrow ad} + x_{b,d \leftarrow ac}  &\leq x_{b}\\
  x_{b,b \leftarrow a} + x_{b,b \leftarrow c} + x_{b,b \leftarrow d} +
  x_{b,b \leftarrow ac} & \nonumber \\ \label{eq:b-b} \tag{b-b} + x_{b,b \leftarrow ad} +
  x_{b,b \leftarrow cd} + x_{b,b \leftarrow acd} &\leq x_{b}
\end{align}
\begin{align}
\label{eq:c-a} \tag{c-a}  x_{c,a \leftarrow b} + x_{c,a \leftarrow d} + x_{c,a \leftarrow bd} &\leq x_{c}\\
\label{eq:c-b} \tag{c-b}  x_{c,b \leftarrow a} + x_{c,b \leftarrow d} + x_{c,b \leftarrow ad} &\leq x_{c}\\
\label{eq:c-d} \tag{c-d}  x_{c,d \leftarrow a} + x_{c,d \leftarrow b} + x_{c,d \leftarrow ab} &\leq x_{c}\\
\label{eq:c-ab} \tag{c-ab}  x_{c,a \leftarrow b} + x_{c,a \leftarrow bd} + x_{c,b \leftarrow a} + x_{c,b \leftarrow ad} &\leq x_{c}\\
\label{eq:c-ad} \tag{c-ad}  x_{c,a \leftarrow d} + x_{c,a \leftarrow bd} + x_{c,d \leftarrow a} + x_{c,d \leftarrow ab} &\leq x_{c}\\
\label{eq:c-bd} \tag{c-bd}  x_{c,b \leftarrow d} + x_{c,b \leftarrow ad} + x_{c,d \leftarrow b} + x_{c,d \leftarrow ab} &\leq x_{c}\\
  x_{c,a \leftarrow b} + x_{c,a \leftarrow d} + x_{c,a \leftarrow bd} + 
  x_{c,b \leftarrow a} + x_{c,b \leftarrow d} + x_{c,b \leftarrow ad} & \nonumber \\ 
\label{eq:c-abd} \tag{c-abd}  + x_{c,d \leftarrow a} + x_{c,d \leftarrow b} + x_{c,d \leftarrow ab}  &\leq 2x_{c}\\
\label{eq:c-2-abd} \tag{c-2-abd}  x_{c,a \leftarrow bd} + x_{c,b \leftarrow ad} + x_{c,d \leftarrow ab}  &\leq x_{c}\\
  x_{c,c \leftarrow a} + x_{c,c \leftarrow b} + x_{c,c \leftarrow d} +
x_{c,c \leftarrow ab} & \nonumber \\ \label{eq:c-c} \tag{c-c} + x_{c,c \leftarrow ad} + x_{c,c
\leftarrow bd} + x_{c,c \leftarrow abd} &\leq x_{c}
\end{align}
\begin{align}
\label{eq:d-a} \tag{d-a}  x_{d,a \leftarrow b} + x_{d,a \leftarrow c} + x_{d,a \leftarrow bc} &\leq x_{d}\\
\label{eq:d-b} \tag{d-b}  x_{d,b \leftarrow a} + x_{d,b \leftarrow c} + x_{d,b \leftarrow ac} &\leq x_{d}\\
\label{eq:d-c} \tag{d-c}  x_{d,c \leftarrow a} + x_{d,c \leftarrow a} + x_{d,c \leftarrow ab} &\leq x_{d}\\
\label{eq:d-ab} \tag{d-ab}  x_{d,a \leftarrow b} + x_{d,a \leftarrow bc} + x_{d,b \leftarrow a} + x_{d,b \leftarrow ac} &\leq x_{d}\\
\label{eq:d-ac} \tag{d-ac}  x_{d,a \leftarrow c} + x_{d,a \leftarrow bc} + x_{d,c \leftarrow a} + x_{d,c \leftarrow ab} &\leq x_{d}\\
\label{eq:d-cd} \tag{d-cd}  x_{d,b \leftarrow c} + x_{d,b \leftarrow ac} + x_{d,c \leftarrow b} + x_{d,c \leftarrow ab} &\leq x_{d}\\
  x_{d,a \leftarrow b} + x_{d,a \leftarrow c} + x_{d,a \leftarrow bc} + 
  x_{d,b \leftarrow a} + x_{d,b \leftarrow c} + x_{d,b \leftarrow ac} & \nonumber \\ 
\label{eq:d-abc} \tag{d-abc}  + x_{d,c \leftarrow a} + x_{d,c \leftarrow b} + x_{d,c \leftarrow ab}  &\leq 2x_{d}\\
\label{eq:d-2-abd} \tag{d-2-abd}  x_{d,a \leftarrow bc} + x_{d,b \leftarrow ac} + x_{d,c \leftarrow
    ab}  &\leq x_{d}\\
  x_{d,d \leftarrow a} + x_{d,d \leftarrow b} + x_{d,d \leftarrow c} +
  x_{d,d \leftarrow ab} & \nonumber \\ \label{eq:d-d} \tag{d-d} + x_{d,d \leftarrow ac} +
  x_{d,d \leftarrow bc} + x_{d,d \leftarrow abc} &\leq x_{d}
\end{align}

Using (\ref{eq:reda}--\ref{eq:redd}) it is possible to
eliminate the variables $x_{a,a \leftarrow bcd}$, $x_{b,b \leftarrow
  acd}$, $x_{c,c \leftarrow abd}$, and $x_{d,d \leftarrow abc}$, and
(\ref{eq:reda}--\ref{eq:redd}) from the representation. This
leaves us with a representation of $\fvpolyy{\{a,b,c,d\},{\cal P}_{\{a,b,c,d\}}}$ using $4 +
28 + 4 \times (9+6) = 92$ variables, 25 equations, 36 inequalities of
type (\ref{eq:altfacet}), four lower bounds on the variables $x_{\varj}$,
56 lower bounds (of 0) on the variables
$x_{\vari, \varj \leftarrow \varsubsetj}$ where $|\varsubsetj|<3$, and
four lower bounds (of 0) on the variables $x_{\vari \leftarrow
  \varsubsetj}$ where $|\varsubsetj|=3$. In total we have 100
inequalities.

We have given an explicit extended representation of
$\fvpolyy{\{a,b,c,d\},{\cal P}_{\{a,b,c,d\}}}$. Here is that representation described more
briefly.
\begin{itemize}
\item   $x_{a} + x_{b} + x_{c} + x_{d} = 1$.
\item 24/2 = 12 unique permutations of   $x_{a \leftarrow b} =  x_{a,a \leftarrow b} +  x_{c,a \leftarrow b}  +  x_{d,a \leftarrow b}$.
\item 24/2 = 12 unique permutations of   $x_{a \leftarrow bc} =  x_{a,a \leftarrow bc} +  x_{d,a \leftarrow bc}$.
\item 24/2 = 12 unique permutations of   $x_{a,b \leftarrow c} + x_{a,b
    \leftarrow d} + x_{a,b \leftarrow cd} \leq  x_{a}$.
\item 24/2 = 12 unique permutations of  $x_{a,b \leftarrow c} + x_{a,b \leftarrow cd} + x_{a,c \leftarrow b} + x_{a,c \leftarrow bd}\leq x_{a}$.
\item 24/6 = 4 unique permutations of   $x_{a,b \leftarrow c} + x_{a,b \leftarrow d} + x_{a,b \leftarrow cd}
  +
  x_{a,c \leftarrow b} + x_{a,c \leftarrow d} + x_{a,c \leftarrow bd} 
  + x_{a,d \leftarrow b} + x_{a,d \leftarrow c} + x_{a,d \leftarrow
    cb}  \leq 2x_{a}$.
\item 24/6 = 4 unique permutations of  $x_{a,b \leftarrow cd} + x_{a,c \leftarrow bd} + x_{a,d \leftarrow
    bc}  \leq x_{a}$.
\item 24/6 = 4 unique permutations of  $x_{a,a \leftarrow b} + x_{a,a \leftarrow c} + x_{a,a \leftarrow d} +
  x_{a,a \leftarrow bc}  + x_{a,a \leftarrow bd} +
  x_{a,a \leftarrow cd} + x_{a \leftarrow bcd} \leq x_{a}$.
\item 4 lower bounds on the variables $x_{\varj}$.
\item 56 lower bounds on variables $x_{\vari, \varj
  \leftarrow \varsubsetj}$ where $|\varsubsetj|<3$.
\item 4 lower bounds on variables $x_{\vari \leftarrow
  \varsubsetj}$ where $|\varsubsetj|=3$.
\end{itemize}

The crucial point is that the convex hull of solutions to our extended
representation can be found by simply dropping the integrality
restrictions on variables. (See \cite[p. 71]{conforti14:integ-progr}
for the relevant proof.) If we `project away' the additional variables
from this convex hull we end up with $\fvpolyy{\vertices,\ppsvalone} =
\conv\left( \bigcup_{\varj \in
    \vertices}\fvpolyys{\vertices}{\varj}\right)$.

We now show explicitly that the facet-defining inequalities of $\fvpolyy{\{a,b,c,d\},{\cal P}_{\{a,b,c,d\}}}$ can be derived
by projection from our extended representation. This projection is
done by forming linear combinations of extended representation
facet-defining inequalities
which only contain `normal' family variables $x_{\vari \leftarrow \varsubsetj}$.

For example, consider adding the following inequalities:
(\ref{eq:a-a}), (\ref{eq:a-2-bcd}),
(\ref{eq:b-b}), (\ref{eq:b-2-acd}),
(\ref{eq:c-c}), (\ref{eq:c-ab}), 
(\ref{eq:d-d}) and  (\ref{eq:d-ab}). 
Note that the RHS of this inequality is 
$x_{a} + x_{a} +
x_{b} + x_{b} +
x_{c} + x_{c} +
x_{d} + x_{d} = 2$.
So the result is
\begin{align*}\label{eq:proj4b}
&x_{a,a \leftarrow b} + x_{a,a \leftarrow c} + x_{a,a \leftarrow d} +
  x_{a,a \leftarrow bc} + x_{a,a \leftarrow bd} +
  x_{a,a \leftarrow cd} + x_{a \leftarrow bcd} \\
+ & x_{a,b \leftarrow cd} + x_{a,c \leftarrow bd} + x_{a,d \leftarrow
    bc} \\
+ & x_{b,b \leftarrow a} + x_{b,b \leftarrow c} + x_{b,b \leftarrow d} +
  x_{b,b \leftarrow ac}  + x_{b,b \leftarrow ad} +
  x_{b,b \leftarrow cd} + x_{b \leftarrow acd} \\
+ & x_{b,a \leftarrow cd} + x_{b,c \leftarrow ad} + x_{b,d \leftarrow
  ac} \\
+ & x_{c,c \leftarrow a} + x_{c,c \leftarrow b} + x_{c,c \leftarrow d} +
x_{c,c \leftarrow ab} + x_{c,c \leftarrow ad} + x_{c,c
\leftarrow bd} + x_{c \leftarrow abd} \\
+ & x_{c,a \leftarrow b} + x_{c,a \leftarrow bd} + x_{c,b \leftarrow
  a} + x_{c,b \leftarrow ad} \\
+ & x_{d,d \leftarrow a} + x_{d,d \leftarrow b} + x_{d,d \leftarrow c} +
  x_{d,d \leftarrow ab}  + x_{d,d \leftarrow ac} +
  x_{d,d \leftarrow bc} + x_{d \leftarrow abc} \\
+ & x_{d,a \leftarrow b} + x_{d,a \leftarrow bc} + x_{d,b \leftarrow
  a} + x_{d,b \leftarrow ac} \leq 2.
\end{align*}
Using (\ref{eq:eq1}-\ref{eq:eq24}) we can simplify this to
\begin{align}
&x_{a \leftarrow b} + x_{a,a \leftarrow c} + x_{a,a \leftarrow d} +
  x_{a \leftarrow bc} + x_{a \leftarrow bd} +
  x_{a \leftarrow cd} + x_{a \leftarrow bcd} \nonumber \\
+ & x_{a,b \leftarrow cd} \nonumber \\
+ & x_{b \leftarrow a} + x_{b,b \leftarrow c} + x_{b,b \leftarrow d} +
  x_{b \leftarrow ac}  + x_{b \leftarrow ad} +
  x_{b,b \leftarrow cd} + x_{b \leftarrow acd} \nonumber \\
+ & \nonumber \\
+ & x_{c,c \leftarrow a} + x_{c,c \leftarrow b} + x_{c,c \leftarrow d} +
x_{c,c \leftarrow ab} + x_{c \leftarrow ad} + x_{c
\leftarrow bd} + x_{c \leftarrow abd} \nonumber \\
+ & \nonumber \\
+ & x_{d,d \leftarrow a} + x_{d,d \leftarrow b} + x_{d,d \leftarrow c} +
  x_{d,d \leftarrow ab}  + x_{d \leftarrow ac} +
  x_{d \leftarrow bc} + x_{d \leftarrow abc} \nonumber \\
+ &  \leq 2.
\end{align}
This inequality can then be weakened by adding the lower bounds for
the 14 remaining extended variables (thus removing them) which results
in the 4B facet (\ref{eq:4b}) of $\fvpolyy{\{a,b,c,d\},{\cal P}_{\{a,b,c,d\}}}$.

We now show how each of the facet classes 4B-4J for $\fvpolyy{\{a,b,c,d\},{\cal P}_{\{a,b,c,d\}}}$
listed in Section~\ref{sec:v4} can be derived by projection from the
extended representation. Projection is achieved by multiplying each
facet-defining inequality in the extended representation by a non-negative scalar. Let the
vector of these scalars be denoted $u \geq 0$. In the following list
we only provide positive components of $u$ and do not bother to list
those components of $u$ corresponding to variable lower bounds. (Note
that since these $u$ vectors generate \emph{facet-defining} inequalities of
$\fvpolyy{\{a,b,c,d\},{\cal P}_{\{a,b,c,d\}}}$, they must be \emph{extreme} rays of the relevant projection
cone \cite{balas05:_projec_liftin_exten_formul_integ_combin_optim}.)

\begin{description}
\item [4B facet] 
\ \\
$u_{a-a} = 1, u_{a-2-bcd}=1,
 u_{b-b} = 1, u_{b-2-acd}=1,
 u_{c-c} = 1, u_{c-ab}=1,
 u_{d-d} = 1, u_{d-ab}=1$
\item[4C facet]
\ \\
$u_{a-a}=1, u_{a-2-bcd} =1,
u_{b-b}=1, u_{b-a} = 1,
u_{c-c}=1, u_{c-ad} = 1,
u_{d-d}=1, u_{d-ac} = 1$
\item[4D facet]
\ \\
$u_{a-a}=2, u_{a-b} =1,
u_{b-b}=1, u_{b-ac}=1, u_{b-ad}=1,
u_{c-c}=1, u_{c-abd}=1,
u_{d-d}=1, u_{d-abc}=1$
\item[4E facet]
\ \\
$u_{a-a} = 2, 
u_{b-b} = 1, u_{b-2-acd} = 1,   
u_{c-c} = 1, u_{c-2-abd} = 1,   
u_{d-d} = 1, u_{d-2-abc} = 1$
\item[4F facet]
\ \\
$u_{a-a}=1,u_{a-bcd}=1,
u_{b-b}=1,u_{b-acd}=1,
u_{c-c}=2,u_{c-d}=1,
u_{d-d}=2,u_{d-c}=1$
\item[4G facet]
\ \\
$u_{a-a}=1, u_{a-bcd}=1,
u_{b-b}=1, u_{b-ad}=1, u_{b-cd}=1,
u_{c-c}=2, u_{c-d}=1,
u_{d-d}=2, u_{d-bc}=1$
\item[4H facet]
\ \\
$u_{a-a}=2,
u_{b-b}=1, u_{b-a}=1,
u_{c-c}=1, u_{c-ad}=1,
u_{d-d}=1, u_{d-ac}=1$
\item[4I facet]
\ \\
$u_{a-a}=2, u_{a-bcd}=1, u_{b-b}=2, u_{b-acd}=1, u_{c-c}=2,
u_{c-ad}=1, u_{c-bd} = 1, u_{d-d}=2,
u_{d-ac}=1, u_{d-bc} = 1$
\item[4J facet]
\ \\
$u_{a-a}=2, u_{a-2-bcd}=1, 
u_{b-b}=1, u_{b-ac}=1, u_{b-ad}=1,
u_{c-c}=1, u_{c-ab}=1, u_{c-ad}=1,
u_{d-d}=1, u_{d-ab}=1, u_{d-ac}=1$
\end{description}

We have shown how to generate all facets of $\fvpolyy{\vertices,\ppsvalone}$ for
$|\vertices| = 4$ from the $|\vertices| = 3$ case. This was done by
constructing the desired convex hull using an extended representation
and then projecting away the extraneous variables. Although in this
case we already had the convex hull for $|\vertices| = 4$ (by direct
computation using \texttt{cdd}) it is clear that the same technique
could be used to construct the convex hull for $|\vertices| = 5$ and
above. The difficulty with this approach is identifying which
projections $u \geq 0$ generate facets. It was noted above that we can
restrict attention to $u$ which are extreme rays of the relevant
projection cone. However, in general, not all extreme rays generate
facets, it also necessary that the number of dimensions `lost' when
projecting the entire polytope matches the number lost when projecting
the face whose projection is the putative facet
\cite{balas05:_projec_liftin_exten_formul_integ_combin_optim}.
We do not investigate this here, leaving this issue for future work.
\end{document}